\documentclass[a4paper,11pt]{article}
  
\usepackage[inline]{enumitem}
\usepackage[margin=1in]{geometry}
\usepackage{amsmath,amssymb,amsthm,bm,dsfont}
\usepackage{mathtools,graphicx}
\usepackage{makecell,longtable}
\usepackage{algorithm,algpseudocode}

\usepackage[capitalize,nameinlink]{cleveref}
\usepackage[authoryear]{natbib}
 \bibpunct[, ]{(}{)}{,}{a}{}{,}%
 \def\newblock{\ }%

\DeclareMathOperator{\sign}{sign}

\newtheorem{theorem}{Theorem}
\newtheorem{lemma}{Lemma}

\newtheorem{proposition}{Proposition}
\newtheorem{assumption}{Assumption}
\newtheorem{definition}{Definition}

\newtheorem{remark}{Remark}
\newtheorem{example}{Example}

\newcommand{\abs}[1]{\lvert#1\rvert}

\newcommand{\babs}[1]{\bigl\lvert#1\bigr\rvert}

\newcommand{\babss}[1]{\biggl\lvert#1\biggr\rvert}

\newcommand{\Bnorm}[1]{\Bigl\lVert#1\Bigr\rVert}

\DeclareMathOperator*{\argmax}{arg\,max}

\newcommand{\ind}[1]{\mathds{1}\left\lbrace #1 \right\rbrace}

\DeclarePairedDelimiterX\Basics[1](){ #1}

\def\EE{{\mathbb{E}}}\def\PP{{\mathbb{P}}}
\def\RR{\mathbb{R}}\def\NN{\mathbb{N}}

\begin{document}
\title{Sample complexity of partition identification using multi-armed bandits}

\author{\vspace{0.1in} Sandeep  Juneja \\
 \vspace{0.1in}
Subhashini Krishnasamy  \\
TIFR, Mumbai}

\date{\today}

\maketitle

\begin{abstract}%
Given a vector of probability distributions, or arms, each of which can be sampled independently, we consider the problem of identifying the partition to which this vector belongs from a finitely partitioned universe of such vector of distributions. We study this as a pure exploration problem in multi armed bandit settings and develop sample complexity bounds on the total mean number of samples required for identifying the correct partition with high probability.  This framework subsumes  well studied problems such as finding the best arm or the best few arms.
We consider distributions belonging to the single parameter exponential family and
primarily consider partitions  where the vector of means of arms lie either in a given  set or its complement. The sets considered  correspond to distributions where there exists  a mean above a specified threshold, where the set is a half space and where either the set or its complement  is a polytope, or more generally, a convex set. 
  In these settings, we characterize the lower bounds on mean number of samples for each arm  highlighting 
  their dependence on the problem geometry. Further, inspired by the lower bounds, we propose algorithms  that can match these bounds asymptotically with decreasing probability of error.  Applications of this framework may be diverse. 
  We briefly discuss one associated with  finance.
\end{abstract}

\section{Introduction}
Suppose that  $\Omega$ denotes a collection of  vectors $ \nu =(\nu_1, \ldots, \nu_K)$ where each $\nu_i$ is a probability distribution on $\RR$. Further,  $\Omega= \cup_{i=1}^m \mathcal{A}_i$ where the component sets $\mathcal{A}_i$ are disjoint, and thus partition
$\Omega$.  In this set-up,  given $ \mu =(\mu_1, \ldots, \mu_K) \in \Omega$, we consider the problem of
identifying the correct component  $\mathcal{A}_i$ that contains $\mu$.
The distributions  $(\mu_i: i \leq K)$ are not known to us, however, it is possible to generate independent samples from each $\mu_i$. We call this the partition identification or $\mathcal{PI}$ problem.

  We consider algorithms that sequentially and adaptively generate samples from each distribution in  $\mu$  and then after generating finitely many samples,  stop  and announce a component of $\Omega$  that is inferred to contain $\mu$. Specifically,  we study the so-called $\delta$-PAC algorithms in the ${\mathcal PI}$ framework. As is well known, PAC algorithms stands for {\em probably approximately correct} algorithms.

 \begin{definition}
  {\em  An algorithm is said to be
{\em $\delta$-PAC} for the $\mathcal{PI}$ problem $\Omega= \cup_{i=1}^m \mathcal{A}_i$,  if, for every $\mu \in \Omega$, 
for any specified $\delta \in (0,1)$, it 
 restricts  the  probability of announcing  an incorrect component  to at most $\delta$. }
 \end{definition}

More generally, in similar sequential decision making problems, algorithms are said to provide
$\delta$-PAC guarantees if the probability of incorrect decision is bounded from above by $\delta$
for each $\delta \in (0,1)$.

In multi-armed bandit (MAB) literature, for any $\nu \in \Omega$, generating a sample from distribution $\nu_i$ is referred to as sampling from, or pulling,  an arm $i$. 
The $\mathcal{PI}$ framework is quite general and captures   popular pure exploration
problems studied in the MAB literature. 
For instance, the problem of
finding the best arm, that is, the arm with the highest mean,  is well studied and fits $\mathcal{PI}$ framework (see, e.g.,  in learning theory \cite{garivier2016optimal},  \cite{kaufmann2016complexity},  \cite{russo2016simple}, \cite{jamieson2014lil},     
\cite{bubeck2011pure},   \cite{audibert2010best},
\cite{even2006action}, \cite{mannor2004sample}; in earlier statistics literature - \cite{jennison1982asymptotically}, \cite{bechhofer1968sequential}, \cite{paulson1964sequential}; in simulation theory   literature
- \cite{glynn2004large},  \cite{kim2001fully}, \cite{chen2000simulation},
\cite{dai1996converge}, \cite{ho1992ordinal}). 

 More generally, identifying $r$ arms (for some $r<K$) with the the largest $r$ means amongst $K$ distributions
 also is a  $\mathcal{PI}$ problem ( see, e.g., \cite{kaufmann2013information}, \cite{kalyanakrishnan2012pac}).

Sample complexity of an algorithm is defined as the expected total number of arms pulled by the algorithm  before it terminates. 
Further,  $\delta$-PAC guarantees provided by algorithms  impose constraints on expected number of times
each arm must be pulled. These constraints are made explicit using the `transportation  inequality'  developed by \cite{garivier2016optimal}. (Their work in turn is built upon `change of measure' based  earlier analysis that goes back at least to  \cite{lai1985asymptotically}. Also see,  \cite{mannor2004sample}, \cite{burnetas1996optimal}). 
The transportation  inequality
  allows us  to formulate the problem of arriving at {\em efficient}  lower bounds on sample complexity  in the $\mathcal{PI}$  framework as an optimization problem -  a linear program with infinitely many constraints; this also has  an equivalent max-min formulation. We refer to
 the resulting optimization problem as the {\em lower bound problem}. 
 
The advantage of $\mathcal{PI}$  framework  is that it provides a unified approach to tackle a large class of  problems, both in developing efficient lower bounds on the sample complexity of $\delta$-PAC algorithms, as well as in arriving at 
 $\delta$-PAC algorithms with sample complexity that asymptotically (as $\delta \rightarrow 0$) matches  the developed
  lower bounds under certain distributional restrictions on the arms.

To further analyze the lower bound problem,  we assume that each arm distribution belongs to a single parameter exponential family (SPEF).
See, e.g., \cite{cappe2013kullback}, \cite{garivier2016optimal}, \cite{kaufmann2016complexity},  where similar distributional restrictions 
are imposed.  Examples of SPEF include Binomial, Poisson, Gaussian with known variance, distributions. See, \cite{cappe2013kullback} for an elaborate discussion on SPEF distributions.  Any member of SPEF distribution can be uniquely represented by its mean. This  allows us to consider the partition problem in the parameter space (i.e., $\Omega \subset \RR^K$) instead of the distribution space.
This further allows us to highlight the geometrical structure of the lower bound problem in a relatively simple manner.

We solve the lower bound problem  for SPEF distributions, so that  $\Omega \subset \RR^K$. 
Our focus is primarily on  $\Omega = \mathcal{A}_1 \cup \mathcal{A}_2$,  where we consider  the following settings:

\begin{itemize}
\item {\bf Threshold crossing problem:}
For a threshold $u \in \RR$,
$\mathcal{A}_1 = \{\nu \in \RR^K: \max_{i \leq K} \nu_i >u\}$
and 
 $\mathcal{A}_2 = \{\nu  \in \RR^K:  \max_{i \leq K} \nu_i < u\}$,
 we explicitly solve the lower bound problem for each $\mu \in \Omega$. 
  We refer to this as the threshold crossing problem, and point to the elegant asymmetry in the lower
  bounds depending upon whether $\mu \in \mathcal{A}_1$ or $\mu \in \mathcal{A}_2$.
In Appendix~\ref{sec:app:threshold}, we  briefly discuss  how this problem arises naturally in financial portfolio risk measurement involving nested simulations.

\item {\bf Half-space problem:} 
For specified $(a_1, \ldots, a_K, b) \in \RR^{K+1}$, 
  $\mathcal{A}_1=\{\nu \in \RR^K: \sum_{i=1}^K a_i \nu_i > b\}$ 
and  $\mathcal{A}_2=\{\nu \in \RR^K: \sum_{i=1}^K a_i \nu_i < b\}$,
we characterize the solution to the  lower bound problem for each $\mu \in \Omega$.

\item {\bf Convex set problem: }
When $\mathcal{A}_2$ is a closed convex set and $\mathcal{A}_1$ is its complement in $\RR^K$,
we characterize the solution lower bound problem for each $\mu \in \mathcal{A}_1$. 
We also do this for $\mu \in \mathcal{A}_1$ when $\mathcal{A}_1$ is a polytope, and $\mathcal{A}_2$, a complement of $\mathcal{A}_1$, is a union of half-spaces. 
In these settings we highlight the geometric structure of the problem
and propose geometry based simple algorithms to compute the lower bound solutions. 
\end{itemize}

\cite{garivier2016optimal} solve an equivalent optimization problem  in the best arm setting. 
They further use the solution to arrive at an adaptive $\delta$-PAC algorithm
whose stopping rule is based on the generalized likelihood ratio test earlier proposed in
\cite{chernoff1959sequential}. Also, see \cite{albert61}. 
The sample complexity of the proposed algorithm
is shown to  asymptotically match the lower bound solution (as $\delta \rightarrow 0$).
We note that their  algorithm can be adapted to the problems we consider
to again arrive at an adaptive $\delta$-PAC algorithm
whose sample complexity  asymptotically matches the corresponding lower bound.

The partition identification framework was also considered in
\cite{chernoff1959sequential} where $\Omega$ was restricted to be finite. 
\cite{albert61} generalized this work  to allow for $\Omega$ with infinite elements.
The key difference of our paper compared to these references is that 
we work in a $\delta$-PAC framework that provides explicit error guarantees. Their work involves guarantees with constants 
that are not explicitly available.  Further, our focus is on exploiting the geometry of the solution to the lower bound problem to solve it efficiently.
These issues are not considered in \cite{chernoff1959sequential} and \cite{albert61}.

The rest of the paper is organized as follows: In Section~\ref{sec:problem}, we state the transportation inequality from \cite{kaufmann2016complexity} and state the resultant lower bound problem  in $\mathcal{PI}$  framework as an optimization problem. We also 
spell out preliminaries such as the single parameter exponential family distributions and related assumptions in this section.
In Section~\ref{sec:lb}, we characterize the solution to the lower bound problem for various special cases of partition of $\Omega$ into 
disjoint sets $\mathcal{A}_1$ and $\mathcal{A}_2$. 
For the threshold crossing problem (Section~\ref{sec:threshold-cross}), we give a closed form expression for the solution to the lower bound problem. For the half-space problem (Section~\ref{sec:half-space}), we give a simple characterization of the solution that allows
for easy numerical evaluation.  Similarly, for the problem where $\Omega$ is partitioned into a convex set and its complement, we derive some useful properties of the solution to the lower bound problem (Sections~\ref{sec:convex}, \ref{sec:non-convex}).
 In Section~\ref{sec:algo}, we propose a $\delta$-PAC algorithm  that in substantial generality achieves the derived lower bounds asymptotically as $\delta$ decreases to zero.
The detailed  proofs are given in the appendices.

\section{Preliminaries and basic optimization problem}
\label{sec:problem}

Recall that $\Omega$ denotes a collection of  vectors $ \nu =(\nu_1, \ldots, \nu_K)$ where each $\nu_i$ is a probability distribution in $\RR$. Further,  $\Omega= \cup_{i=1}^m \mathcal{A}_i$ where the $\mathcal{A}_i$ are disjoint, and thus partition
$\Omega$.

Let  {$KL(\mu_i || \nu_i)= \int \log (\frac{d \mu_i}{d \nu_i}(x)) d \mu_i(x)$} denote the Kullback-Leibler divergence
between distributions $\mu_i$ and $\nu_i$. We further assume that for each $\nu, \tilde{\nu} \in \Omega$,
the components $\nu_i$ and $\tilde{\nu}_i$ for each $i$ are mutually absolutely continuous and the expectation
$KL(\nu_i || \tilde{\nu}_i)$ exists (it may be infinite).
For $p,q \in (0,1)$, let
\[
d(p, q) := p \log \left ( \frac{p}{q} \right ) + (1-p)  \log \left ( \frac{1-p}{1-q} \right ),
\]
that is, $d(p, q) $ denotes the KL-divergence between Bernoulli distributions with mean $p$ and $q$, respectively.
For any set $\mathcal B$, let $\mathcal{B}^c$ denote its complement, $\mathcal {B}^o$ its interior, $\bar{\mathcal B}$ its closure and $\partial \mathcal B$ its boundary.

Under a  $\delta$-PAC algorithm, and for $\mu \in \mathcal{A}_j$, the following transportation inequality
 follows from \cite{kaufmann2016complexity}:
\begin{equation}
\label{eq:lb}
\sum_{i=1}^K E_{\mu} N_i \, KL(\mu_i || \nu_i) \geq d(\delta, 1-\delta) \geq \log \left (\frac{1}{2.4 \delta} \right )
\end{equation}
for any $\nu \in \mathcal{A}_j^c$, where $N_i$ denotes the number of times arm $i$ is pulled
by the algorithm. The assumption that
$KL(\nu_i || \tilde{\nu}_i)$ exists, allows the use of Wald's Lemma in proof of  Lemma 1 in \cite{kaufmann2016complexity}.
Taking $t_i = E_{\mu} N_i/\log (\frac{1}{2.4 \delta})$, our lower bound on sample complexity problem can be modelled as the  following convex programming  problem, when  $\mu \in \mathcal{A}_j$ (call it $\bf{O1}$):
\begin{eqnarray*}
    \min_{\mathbf{t} = (t_1, \dots, t_K)} & \sum_{i=1}^K t_i &  \label{eqn:generic_constr0} \\
\mbox{  s.t.  } & \inf_{\nu \in {\mathcal{A}_j^c}} \sum_{i=1}^K t_i KL(\mu_i || \nu_i) \geq  1, & \label{eqn:generic_constr} \\
 & t_i \geq 0 \,\,\, \forall i. & 
\end{eqnarray*}

Letting $w_i= \frac{t_i}{\sum_j t_j}$ and 
$\mathcal{P}_K
\triangleq \{ w \in \RR^k: w_i \geq 0 \, \, \forall i, \sum_{i=1}^K w_i=1 \}$ 
 denote the $K$-dimensional probability simplex, $\bf{O1}$ maybe equivalently stated as
\begin{equation}
\label{eq:lb-problem}
\max_{w \in \mathcal{P}_K} \inf_{\nu \in {\mathcal{A}_j^c}} \sum_{i=1}^K w_i KL(\mu_i || \nu_i).	\tag{Problem \textsf{LB}}
\end{equation}

Let $C^*(\mu)$ be the optimal value of the above problem. The lower bound on the total expected number of samples is then given by $\log (\frac{1}{2.4 \delta}) T^*(\mu)$ where $T^*(\mu) = 1/C^*(\mu)$.

\begin{remark} {\em While 
the optimization problem $\bf{O1}$ is equivalent to Problem \textsf{LB},
one  advantage of the former is that it can be viewed as a linear program with infinitely many constraints, or a semi-infinite linear program (see, e.g., \cite{lopez2007semi}). Then linear programming duality provides a great deal of insight into the solution structure. However, we instead present our analysis using the max-min Problem \textsf{LB}, since   Sion's minimax theorem can be applied on it to directly arrive at the solution.  }
\end{remark}




\noindent {\bf Single Parameter Exponential Families:} In the remaining paper, we consider single parameter exponential family (SPEF) of 
distributions for each arm. For each $1 \leq i\leq K$, let $\rho_i$ denote a reference measure on the real line, and let
\[
\Lambda_i(\eta) \triangleq \log \left (\int_{x \in \RR}\exp(\eta x) d \rho_i(x) \right ).
\]
$\Lambda_i$ is referred to as a cumulant or a log-partition function. 
Further, set
${\mathcal D}_i \triangleq \{ \eta : \Lambda_i(\eta) < \infty \}.$

An SPEF distribution for arm $i$ and $\eta \in {\mathcal D}_i$, $p_{i,\eta}$, has the form
\[
d p_{i,\eta}(x) = \exp(\eta x- \Lambda_i(\eta)) d \rho_i(x).
\]
Note that
$\Lambda_i$ is ${\mathcal C}^{\infty}$ in ${\mathcal D}_i^o$ (see, e.g.,  2.2.24 \cite{dembo38large}). 
Further,  $\Lambda_i(\eta)$ is a convex function
of $\eta \in {\mathcal D}_i^o$, and if the underlying distribution is non-degenerate,  then it is 
strictly convex. 

 
 Let $\Lambda_i^*$ denote the Legendre-Fenchel transform of $\Lambda_i$, that is, 
$ \Lambda_i^*(\theta) = \sup_{\eta\in {\mathcal D}_i} (\eta \theta - \Lambda_i(\eta)).$

Further, let $\mu_i$ denote the mean under $p_{i,\eta_i}$. 
Then,
$\mu_i  =    \Lambda_i'(\eta_i)$ 
for $\eta_i \in {\mathcal D}_i^o$.
In particular, $\mu_i$ is a strictly increasing function of $\eta_i$,
and there is one to one mapping between the two. 
Below we suppress the notational dependence of $\mu_i$ on $\eta_i$
 and vice-versa.

Let 
${\mathcal U}_i \triangleq \{  \Lambda_i'(\eta_i),  \eta_i  \in {\mathcal D}^o_i \}.$ 
Since  $\Lambda_i'(\eta_i)$ is strictly increasing for   $\eta_i  \in {\mathcal D}^o_i$,  ${\mathcal U}_i$ is an open interval, and sans the boundary cases, denotes the value of means attainable for arm $i$. For $\eta_i \in {\mathcal D}_i^o$, the following are well known and easily checked:
$\eta_i =  {\Lambda_i^*}'(\mu_i)$, and
\begin{eqnarray}
\Lambda^*_i(\mu_i) + \Lambda_i(\eta_i) & = & \mu_i \eta_i.	\label{eq:convex-conj}
\end{eqnarray}
For $\eta_i, \beta_i \in {\mathcal D}_i^o$, It is easily seen that 
\[
KL(p_{i,\eta_i} || p_{i,\beta_i})  = 
\Lambda_i(\beta_i) - \Lambda_i(\eta_i) - \mu_i(\beta_i-\eta_i)
 \] 
 where again $\mu_i= \Lambda_i'(\eta_i)$. 
We denote the above by $K_i(\mu_i | \nu_i)$
with $\nu_i= \Lambda_i'(\beta_i)$
emphasizing that when the two distributions are from the same SPEF, Kullback-Leibler divergence 
only depends on the mean values of the distributions. Using \eqref{eq:convex-conj}, we have
\begin{equation}
\label{eq:kl-exp-fam}
K_i(\mu_i | \nu_i) = \Lambda^*_i(\mu_i) - \Lambda^*_i(\nu_i) - \beta_i(\mu_i - \nu_i),
\end{equation}
where $\beta_i = {\Lambda^*_i}'(\nu_i)$. Again, it can be shown that $\Lambda^*_i$ is ${\mathcal C}^{\infty}$ in ${\mathcal U}_i$ (see,  2.2.24 \cite{dembo38large}), and it is strictly convex if $\Lambda_i$ is strictly convex. Thus, $K_i$ is ${\mathcal C}^{\infty}$  in ${\mathcal U}_i$ with respect to each of its arguments.

In the remaining paper, \ref{eq:lb-problem} refers to
\begin{equation}
\label{eq:lb-problem2}
\max_{w \in \mathcal{P}_K} \inf_{\nu \in {\mathcal{A}_j^c}} \sum_{i=1}^K w_i KL(\mu_i | \nu_i),	
\end{equation}
 each $\mathcal{A}_k$ a subset of $\RR^K$, and again $\mu \in \mathcal{A}_j$.

\begin{remark} \label{rem:33} {\em
For any  $w \in \mathcal{P}_K$, the sub-problem  in \textsf{LB}, $\inf_{\nu \in {\mathcal{A}_j^c}} \sum_{i=1}^K w_i KL(\mu_i | \nu_i)$,  has an elegant geometrical interpretation.
For $c>0$, consider the sublevel set
\[
S(\mu,w, c) \triangleq \left\lbrace \nu : \sum_{i=1}^K w_i KL(\mu_i | \nu_i) \leq c \right\rbrace.
\]
Then, for element-wise strictly positive $w$,  $S(\mu,w, 0)= \{\mu\}$. The set $S(\mu,w, c)$ for 
some $c>0$ intersects with $\mathcal{A}_j^c$. Further, the set shrinks as $c$ reduces. We are looking for the smallest $c=c^*$ 
 for which $S(\mu,w, c)$ has a non-empty intersection with $\bar{\mathcal{A}}_j^c$. Equivalently, 
 we are looking for the first $c >0$ for which the set grows beyond the interior of $\mathcal{A}_j$ and intersects with $\bar{\mathcal{A}}_j^c$.  
Thus, 
\[
\inf_{\nu \in \bar{\mathcal{A}}_j^c} \sum_{i=1}^K w_i KL(\mu_i | \nu_i) = \inf\{c: S(\mu,w, c) \cap \bar{\mathcal{A}}_j^c \ne \emptyset\}.
\] 
Figure~\ref{fig:1} demonstrates this in a simple setting of two arms. Arm $i$, for $i=1,2$, is Gaussian distributed with mean $\mu_i=0$ and variance 1.
$\mathcal{A}_1 = \{(\nu_1, \nu_2) \in \RR^2: a_1 \nu_1 + a_2 \nu_2 <b \}$ for $a_1, a_2, b >0$, and it contains $(\mu_1, \mu_2)=(0,0)$. 
$KL(\mu_i |\nu_i) = \frac{\nu_i^2}{2}$ for $i=1,2$. The convex set $S(\mu,w, c)$ is tangential to $a_1 \nu_1 + a_2 \nu_2 =b$ at $c=c^*$.
}
\end{remark}

\begin{figure} 
\begin{center}
\includegraphics[height=0.2\textheight]{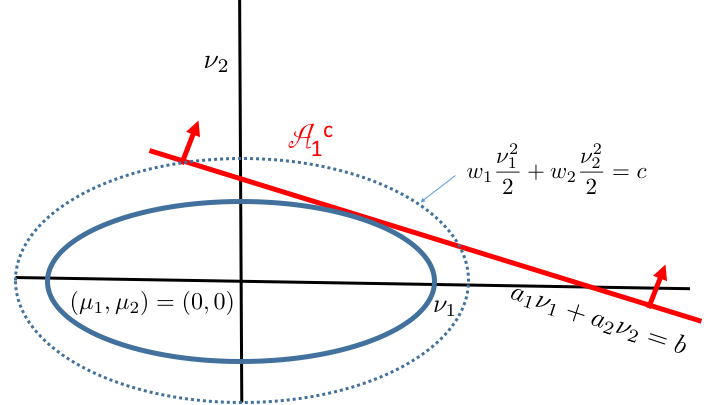}
\end{center}
\caption{Geometrical view of the sub-problem  in \textsf{LB}. Two arms with standard Gaussian distribution. $\mathcal{A}_1$ is a half-space in $\RR^2$,
and $\mu=(0,0) \in \mathcal{A}_1$.}
\label{fig:1}
\end{figure}

\noindent {\bf Conditions on  KL-Divergence:} Since $\Lambda^*_i$ is a convex function, we have that $K_i$ is convex in its first argument.  Since $K_i(\mu_i| \nu_i)$ decreases with $\nu_i$ 
for $\nu_i \leq \mu_i$, and it increases with $\nu_i$ for $\nu_i \geq \mu_i$, it is a quasi-convex function of $\nu_i$.  
For many known SPEFs, for instance, Bernoulli, Poisson and Gaussian with known variance, the KL-divergence is also strictly convex in the second argument. But there are also SPEFs for which it is not convex in the second argument, for e.g., Rayleigh, centered Laplacian and negative Binomial (with number of failures fixed).

Our analysis is substantially simplified when $\sum_{i=1}^K w_i KL(\mu_i | \nu_i)$ is a strictly convex function of $\nu$. 
This is ensured by Assumption~\ref{ass:kl-cvx}:
\begin{assumption} 
 \label{ass:kl-cvx}
 For each $i$, ${\mathcal D}_i^o$ is non-empty and $\Lambda_i(\eta_i)$ is strictly convex for
 $\eta_i \in {\mathcal D}_i^o$. Further,  for any $\mu_i \in {\mathcal U}_i$,  $K_i(\mu_i | \nu_i)$ is a strictly convex function of $\nu_i \in {\mathcal U}_i.$
\end{assumption} 

We also make the following assumption to ease some technicalities.  This assumption holds for most distributions encountered in practice.
\begin{assumption} 
 \label{ass:kl-range}  
For any $\mu_i \in {\mathcal U}_i$, $K_i(\mu_i | \nu_i) \rightarrow \infty$ as $\nu_i \rightarrow \partial {\mathcal U}_i$ with $\nu_i$ taking values in ${\mathcal U}_i$.
\end{assumption}
 

\section{Lower bounds for some $\mathcal{PI}$ problems}
\label{sec:lb}

In this section we explore the structure of \ref{eq:lb-problem} in a number of settings. 
In each setting we specify $\mathcal{A}_1$ and  $\mathcal{A}_2$, and 
$\Omega$ is set to  $\mathcal{A}_1 \cup \mathcal{A}_2$. 

\subsection{Threshold crossing problem }
\label{sec:threshold-cross}
Let
${\mathcal U} = \bigtimes_{i=1}^K  {\mathcal U}_i$, 
 $\mathcal{A}_1= \{\mu \in {\mathcal U}: \max_{i \leq K} \mu_i >u\}$,
 and
 $\mathcal{A}_2 = \{\mu \in {\mathcal U}:  \max_{i \leq K} \mu_i < u\}$.

Theorem \ref{the:theorem1} below points to an interesting asymmetry that arises in the lower bound problem
associated with  threshold crossing as a function $\mu \in \Omega$.
\begin{theorem} \label{the:theorem1}
Suppose that  $(u, \ldots, u) \in {\mathcal U}$. Consider $\mu \in \mathcal{A}_1$ such that, w.l.o.g., for some $i \geq 1$,
\[ \mu_j > u \mbox{ for }  j =2,\ldots, i, \; \mu_j < u \; \mbox{ for } i+1 \leq j \leq K,
\]
and
$K_1(\mu_1 | u) > K_j(\mu_j | u) \; \mbox{ for } j=1, \ldots, i.$
Then, \ref{eq:lb-problem} has a unique solution given by
 \begin{equation} \label{eqn:thresh1}
 w^*_1 = 1, \mbox{ and } w^*_j=0 \mbox{ for } j=2, \ldots, K.
\end{equation}
The
lower bound on expected total number of samples generated equals
$\frac{1}{K_1(\mu_1 | u)} \times \log (\frac{1}{2.4 \delta}).$

When $\mu \in \mathcal{A}_2$, \ref{eq:lb-problem} has a unique solution given by
\begin{equation} \label{eqn:thresh222}
w^*_j \propto  1/K_j(\mu_j | u), \,\, 1 \leq j \leq K,
\end{equation}
 and the
lower bound on expected total number of samples generated equals
$\sum_{j=1}^K \frac{1}{K_j(\mu_j | u)} \times \log (\frac{1}{2.4 \delta}).$
\end{theorem}

Intuitive explanation for the lower bound asymmetry in the two cases
$\mu \in \mathcal{A}_1$ and $\mu \in \mathcal{A}_2$ is as follows:
When $\mu \in \mathcal{A}_1$,
any algorithm has to establish with at least $1-\delta$ probability 
that there exists at least one arm above $u$. The lower bound is then achieved by focussing on the
arm that is most separated from $u$. That is, 
 arm $i$ with $\mu_i>u$ and  with the largest
value of $K_i(\mu_i | u)$. 
 On the other hand, when $\mu \in \mathcal{A}_2$,  any  algorithm would need
to rule that each and every arm has mean less than $u$, again
while controlling the probability of error for each arm.  

In Appendix~\ref{sec:app:threshold}, Example~\ref{ex:1}, we discuss
how the threshold crossing problem arises naturally  in nested simulation used in financial portfolio risk measurement.

\subsection{Half-space problem }
\label{sec:half-space}
 We consider the problem of identifying the half-space to which the mean vector belongs. 
Set
$\mathcal{A}_1 = \{\nu \in    \RR^K \cap {\cal U}: \sum_{k=1}^K a_{k}\nu_k < b \}$ and 
$\mathcal{A}_2 = \{\nu \in  \RR^K \cap {\cal U}: \sum_{k=1}^K a_{k}\nu_k > b \}.$
W.l.o.g. each $a_i$ can be taken to be non-zero and $b>0$.
\ref{eq:lb-problem} may be formulated as:
 For $\mu \in \mathcal{A}_1$, and non-empty $\mathcal{A}_2$, 
 \begin{equation}
 \label{eq:half-space-lb}
\max_{w \in \mathcal{P}_K} \inf_{\nu \in \bar{\mathcal{A}}_2} \sum_{j=1}^K w_j K_j(\mu_j | \nu_j).	
\end{equation}
%

\begin{theorem} \label{thm:Theorem10111}
Under Assumptions~\ref{ass:kl-cvx}, \ref{ass:kl-range}, and that $\mathcal{A}_2$ is non-empty,
there is a unique optimal solution $(w^*, \nu^*)$ to \ref{eq:lb-problem}.
Further,
\begin{equation} \label{eqn:hs00111}
K_i( \mu_i | {\nu}^*_i) = K_1( \mu_1 | {\nu}^*_1)	\quad \forall i,
\end{equation}
\begin{equation} \label{eqn:hs00211}
\sum_{k=1}^K a_{k}{\nu}^*_k = b,
\end{equation}
\begin{equation} \label{eqn:hs00311}
 {\nu}^*_i> \mu_i \mbox{ if }  a_i >0,
\mbox{ and  }
 {\nu}^*_i < \mu_i \mbox{  if } a_i <0.
\end{equation}
Relations (\ref{eqn:hs00111}), (\ref{eqn:hs00211}) and (\ref{eqn:hs00311}) uniquely
specify ${\nu}^* \in {\mathcal U}$.
Moreover,
\begin{equation} \label{eqn:hs00411}
\frac{w^*_i}{a_i}K'_i(\mu_i | {\nu}^*_i) = \frac{w^*_1}{a_1}K'_1(\mu_1 | {\nu}^*_1) \quad \forall i,
\end{equation}
where the derivatives are with respect to the second argument.
\end{theorem}

The proof details are given in the appendix. Ignoring technicalities, the intuition 
for (\ref{eqn:hs00111}) follows from Sion's  Minimax Theorem, which, loosely speaking, implies that
(\ref{eq:half-space-lb}) equals
\[
 \inf_{\substack{\nu \in  \RR^K \cap {\cal U}: \\
 \sum_{k=1}^K a_{k}\nu_k \geq b}} \, \max_{w \in \mathcal{P}_K}
 \sum_{j=1}^K w_j K_j(\mu_j | \nu_j)=
  \inf_{\substack{\nu \in  \RR^K \cap {\cal U}: \\ \sum_{k=1}^K a_{k}\nu_k \geq b}} \, \max_j K_j(\mu_j | \nu_j).
  \]
  Relations (\ref{eqn:hs00111}), (\ref{eqn:hs00211}) and (\ref{eqn:hs00311}) then
  follow from KKT conditions applied to RHS above. Uniqueness of $\nu^*$ follows as $\max_j K_j(\mu_j | \nu_j)$
  is a strictly convex function of $\nu$.    Equation~(\ref{eqn:hs00411}) corresponds to the slope matching that occurs
  as the boundary of the sub-level set associated with $w^*$ (see Remark~\ref{rem:33}) is tangential to the hyperplane
  $\sum_{k=1}^K a_{k}\nu_k = b$.

 \subsection{$\mathcal{A}_2$ is a convex set}
 \label{sec:convex}
 To avoid undue technicalities, assume that $\Omega \subset {\mathcal U}$.
 Suppose that $\mathcal{A}_2 $ is a non-empty closed convex set and $\mu \in \mathcal{A}_1$. 
 Let the associated lower bound problem be denoted by \ref{eq:conv-set-lb}.
\begin{equation}
 \label{eq:conv-set-lb}
\max_{w \in \mathcal{P}_K} \inf_{\nu \in \mathcal{A}_2} \sum_{j=1}^K w_j K_j(\mu_j | \nu_j).	\tag{Problem \textsf{CVX}}
\end{equation}
The solution to \ref{eq:conv-set-lb} and each of its sub-problems  
$\inf_{\nu \in \mathcal{A}_2} \sum_{j=1}^K w_j K_j(\mu_j | \nu_j)$ is finite. This follows as 
for each feasible $w \in \mathcal{P}_K$, and for some $\nu^{(0)} \in \mathcal{A}_2$,
\[
\inf_{\nu \in \mathcal{A}_2} \sum_{j=1}^K w_j K_j(\mu_j | \nu_j)
\leq \sum_{j=1}^K w_j K_j(\mu_j | \nu^{(0)}_j) < \max_j K_j(\mu_j | \nu^{(0)}_j).
\]
Let $C^*$ denote the optimal value for \ref{eq:conv-set-lb}.
Under Assumption~\ref{ass:kl-cvx},  $\sum_{j=1}^K w_j K_j(\mu_j | \cdot)$ is strictly convex
and there is a unique $\nu \in \partial \mathcal{A}_2$  that achieves the minimum in the sub-problem
$\inf_{\nu \in \mathcal{A}_2} \sum_{j=1}^K w_j K_j(\mu_j | \nu_j)$.
Let ${\nu}(w)$ denote this unique solution for any $w \in \mathcal{P}_K$. 
Lemma~\ref{lem:unique-vu*}  below shows that for every optimal solution  to \ref{eq:conv-set-lb}, the same $\nu$ achieves the minimum in the above sub-problem. 
 \begin{lemma}
\label{lem:unique-vu*} Under Assumption~\ref{ass:kl-cvx},
for any $w^*, s^*$ that are optimal for \ref{eq:conv-set-lb},  $\nu(w^*) = \nu(s^*)$.
\end{lemma}

Let $\nu^*$ be the unique value of $\nu$ which achieves the minimum in the sub-problem for every optimal solution. 
In Theorem~\ref{thm:min-max}, we provide an alternate characterization of $\nu^*$, as well as 
a characterization of the solution of \ref{eq:conv-set-lb}.

Some notation is needed to  state Theorem~\ref{thm:min-max}. 
For any index set $\mathcal{J} \subseteq [K]$ and vector $\nu \in \RR^K$, let $\nu_{\mathcal J}$ denote the projection of the vector $\nu$ on to the lower dimensional subspace with coordinate set given by $\mathcal{J}$. 
Similarly, for any set $\mathcal B \subseteq  \RR^K$, let $\mathcal B_{\mathcal J}$ denote its projection onto the subspace restricted to the coordinate set $\mathcal J$, i.e, $\mathcal B_{\mathcal J} = \{\nu_{\mathcal J}: \nu \in \mathcal B \}$. Note that if $\mathcal B$ is convex, then $\mathcal B_{\mathcal J}$ is also convex. If $\mathcal B$ is the c-sublevel set of a convex function $f$, then  $$\mathcal B_{\mathcal J} = \{\nu_{\mathcal J}: f(\nu_{\mathcal J}, \nu_{\mathcal{J}^c}) \leq c \text{ for some } \nu_{\mathcal{J}^c} \in \RR^{\abs{\mathcal{J}^c}} \} = \{\nu_{\mathcal J}: \inf_{ \nu_{\mathcal{J}^c} \in \RR^{\abs{\mathcal{J}^c}}}f(\nu_{\mathcal J}, \nu_{\mathcal{J}^c}) \leq c \}.$$
In other words,  $\mathcal B_{\mathcal J}$ is the c-sublevel set of the function $h_{\mathcal J} := \inf_{ \nu_{\mathcal{J}^c} \in \RR^{\abs{\mathcal{J}^c}}}f(\nu_{\mathcal J}, \nu_{\mathcal{J}^c})$.

\begin{theorem}
\label{thm:min-max11}
Suppose that $\mu \in {\mathcal A}_1$, $\mathcal{A}_2$ is non-empty, $\Omega \subset {\mathcal U}$, and Assumptions  1 and 2 hold. Then, for any optimal solution $(w^*,\nu^*)$  to \ref{eq:conv-set-lb}, the  $\nu^*$ uniquely solves the  min-max problem
\begin{eqnarray}
\label{eq:min-max11} 
\inf_{\nu \in \mathcal{A}_2} \max_i  K_i(\mu_i | \nu_i).
\end{eqnarray}

\noindent Further, the following are necessary and sufficient conditions for such an  
$(w^*,\nu^*)$. Let $\mathcal{I} = \argmax_i K_i(\mu_i | \nu^*_i)$. Then,
\begin{enumerate}
\item \label{item:zero11} $w^*_i = 0 \quad \forall i \in \mathcal{I}^c,$
\item \label{item:boundary11} $\nu^*_{\mathcal I} \in \partial (\mathcal{A}_2)_{\mathcal I}$, and
\item \label{item:kkt11} there exists a supporting hyperplane of $(\mathcal{A}_2)_{\mathcal I}$ at $\nu^*_{\mathcal I}$ given by $\sum_{i \in \mathcal I} a_i \nu_i = b$ such that
\begin{gather}
 {\nu}^*_i> \mu_i \mbox{ if }  a_i >0,
\mbox{ and  }
 {\nu}^*_i < \mu_i \mbox{  if } a_i <0 \quad \forall i \in \mathcal{I},	\label{eq:sign-match11}	\\
\frac{w^*_i}{a_i} K_i'(\mu_i|\nu^*_i) =  \frac{w^*_j}{a_j} K_j'(\mu_j|\nu^*_j)	\quad \forall i, j \in \mathcal{I}.	\label{eq:kkt-half-space11}
\end{gather}
\end{enumerate}
\end{theorem}

Problem \textsf{CVX} (and indeed \ref{eq:lb-problem}) may heuristically be viewed as a game between an optimal algorithm and nature.
An algorithm  picks a $w \in \mathcal{P}_K$ that provides a recipe for proportionate sampling of different arms. Nature then selects a
$\nu \in \mathcal{A}_2$ that for a given $w$ minimizes  $\sum_{j=1}^K w_j K_j(\mu_j | \nu_j)$, and hence for the algorithm is the most difficult 
to separate from $\mu$. The algorithm looks for a $w$ that maximizes this minimum separation. Theorem~\ref{thm:min-max11}
makes an interesting observation that for convex $\mathcal{A}_2$, the algorithm has the option of not sampling some arms. Maximum separation may be obtained by focusing on a subset of arms and showing that they are well separated from the projection of $\mathcal{A}_2$
along the subspace associated with these arms. 

Condition (\ref{item:kkt11}) in Theorem~\ref{thm:min-max11} highlights the fact that along the projected space, finding a solution to
Problem \textsf{CVX} is equivalent to finding a solution to an appropriate half-space problem that is tangential to the projected convex set.

\begin{figure}
\begin{center}
\includegraphics[height=0.2\textheight]{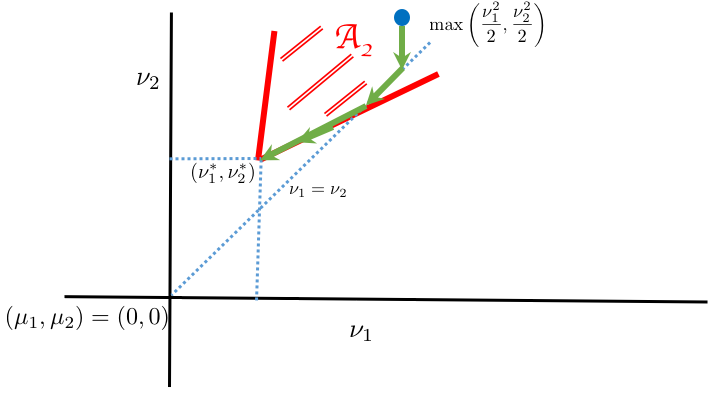}
\end{center}\caption{Algorithm to solve  (\ref{eq:min-max11}) in a simple setting of  two arms with standard Gaussian distribution. $\mathcal{A}_2$ is a closed
convex set and $\mu=(0,0) \in \mathcal{A}_1$. Since $\nu_2^* > \nu_1^*$, $\mathcal{I} =\{2\}$. This suggests that it is optimal to only sample arm 2 
  to separate  $\mu=(0,0)$ from $\mathcal{A}_2$.}
\label{fig:fig2}
\end{figure}

\begin{remark} {\em
 Since $\max_i  K_i(\mu_i | \nu_i)$ is a strictly convex function of $\nu$, (\ref{eq:min-max11}) shows that Problem \textsf{CVX}
 maybe solved for $\nu^*$ using any standard convex programming solver. 
 Remark~\ref{rem:smooth_bound} below emphasizes the point
 that  $w^*$ is easily calculated once $\nu^*$ is known, if there is a unique supporting hyperplane in $\RR^{|{\mathcal I}|}$,  of $(\mathcal{A}_2)_{\mathcal I}$ at $\nu^*_{\mathcal I}$.

 Figure~\ref{fig:fig2} demonstrates how a steepest descent based procedure may work to solve  (\ref{eq:min-max11}) in a simple setting of two arms. Arm $i$, for $i=1,2$, is Gaussian distributed with mean $\mu_i=0$ and variance 1. $KL(\mu_i |\nu_i) = \frac{\nu_i^2}{2}$ for $i=1,2$.
 The algorithm starts at a point $(\nu_1, \nu_2) \in \mathcal{A}_2$ with $\nu_2 >\nu_1$. Steepest descent direction to minimize $\max_{i=1,2} \frac{\nu_i^2}{2}$
 corresponds to reducing $\nu_2$ until $\nu_2 = \nu_1$. It then corresponds to decending along the direction $\nu_1=\nu_2$, until
 boundary of $\mathcal{A}_2$ is hit. In Figure~\ref{fig:fig2}, the algorithm continues to descend along the boundary reducing the 
  value of $\max_{i=1,2} \frac{\nu_i^2}{2}$ until the optimal point $(\nu_1^*, \nu_2^*)$. 
 Since $\nu_2^* > \nu_1^*$, we have $\mathcal{I} =\{2\}$. Thus, the lower bound analysis suggests that it is optimal to only sample arm 2 
  to separate  $\mu=(0,0)$ from $\mathcal{A}_2$.}
  \end{remark}

\begin{remark} \label{rem:smooth_bound} {\em
Condition \ref{item:kkt} shows that the problem has a unique solution, i.e., the optimal $w^*$ is a singleton, if there is a unique supporting hyperplane of $(\mathcal{A}_2)_{\mathcal I}$ at $\nu^*_{\mathcal I}$.
Consider the case where $\mathcal{A}_2 = \{ \nu: f(\nu) \leq c \}$ is the c-sublevel set of a convex function $f$.
Then,  $(\mathcal{A}_2)_{\mathcal I}$ is the c-sublevel set of the function $h: \Re^{|\mathcal{I}|}
\rightarrow \Re, h(\nu_{\mathcal I}) := \inf_{ \nu_{\mathcal{I}^c} \in \RR^{\abs{\mathcal{I}^c}}}f(\nu_{\mathcal I}, \nu_{\mathcal{I}^c})$. 
Further suppose that $h(\cdot)$ is a smooth function. Then, the unique tangential hyperplane 
at $\nu^*_{\mathcal{I}}$ is given by
$\nabla h(\nu^*_{\mathcal{I}})^{\intercal} (\nu_{\mathcal{I}}-\nu^*_{\mathcal{I}})=0.$
In particular, in this case for $i \in \mathcal{I}$,
$w^*_i \propto \frac{\frac{\partial h}{\partial \nu_i}(\nu^*_{\mathcal{I}})}{K_i'(\mu_i|\nu^*_i)}.$
}
\end{remark}


\vspace{0.1in}

\subsection{ $\mathcal{A}_1$ is a polytope}
\label{sec:non-convex}
In Section~\ref{sec:convex}, $\mathcal{A}_2$ is convex, while $\mathcal{A}_1$ need not be.
This allowed us to explicitly characterize the solution to the lower bound problem. We now briefly
consider the case where $\mathcal{A}_1$ is convex, and  $\mathcal{A}_2$ need not be. Specifically, we examine the case where
$\mathcal{A}_1$ is a polytope so that  $\mathcal{A}_2$ is a union of half-spaces.
Just as the single half-space problem was useful in studying
 the case where $\mathcal{A}_2$ is convex,
analyzing $\mathcal{A}_2$ when it is a union of half-spaces,  may provide  insights to a more general problem  where $\mathcal{A}_2$  is a union of convex sets.
The latter may be an interesting area for future research.

Let
\begin{equation} \label{eqn:a700}
\mathcal{B}_j \triangleq \{\nu \in \RR^K: \sum_{k=1}^K a_{j,k}\nu_k \geq b_j \},
\end{equation}
 each $b_j \geq 0$, and $\mathcal{A}_2 = \cup_{j=1}^m \mathcal{B}_j$ be the union of these half-spaces. 
 To ease technicalities, suppose that ${\mathcal U} = \RR^K$. The lower  bound problem may be expressed as 
 \begin{equation}
 \label{eq:union-lb}
C^*(\mu) =  \max_{w \in \mathcal{P}_K} \inf_{\nu \in \cup_{j=1}^m \mathcal{B}_j} \sum_{i=1}^K w_i K_i(\mu_i | \nu_i) .
 \end{equation}
 
Let  $\mathcal{W}(\mu)$ denote the optimal solution set. Lemma~\ref{lem:unique-soln} shows that  the optimization problem in \eqref{eq:union-lb} has a unique solution, that is,
$\mathcal{W}(\mu)$ is a singleton.

\begin{lemma}
\label{lem:unique-soln}
There is a unique $w \in \mathcal{P}_K$ that achieves the maximum in \eqref{eq:union-lb}.
\end{lemma}

 \begin{remark} {\em
 It  is easy to see that the {\bf best arm identification problem} is a special case of this problem.   To see  this, suppose arm $1$ has the highest mean among the $K$ arms, i.e., $\mu_1 \geq \mu_j \; \forall j \neq 1.$ We then have $\mathcal{A}_2 = \cup_{j=2}^K \mathcal{B}_j,$ where for any $j$, $\mathcal{B}_j = \{\nu \in \RR^K: \nu_j \geq \nu_1\}.$}
\end{remark}

Observe that $\inf_{\nu \in \mathcal{A}_2} \sum_{i=1}^K w_i K_i(\mu_i | \nu_i)$ being an infimum of linear functions of $w$,  is a concave function 
of $w$, for any $\mathcal{A}_2$. Thus, standard gradient descent methods can be used to solve (\ref{eq:union-lb}),
once an algorithm exists for solving  $ g(\mu, w) \triangleq   \inf_{\nu \in \cup_{j=2}^m  \mathcal{B}_j} \sum_{i=1}^K w_i K_i(\mu_i | \nu_i)$.
This is straightforward as  
\[
\inf_{\nu \in \cup_{j=2}^m  \mathcal{B}_j} \sum_{i=1}^K w_i K_i(\mu_i | \nu_i)=
 \min_{j \leq m} \inf_{\nu \in \mathcal{B}_j} \sum_{i=1}^K w_i K_i(\mu_i | \nu_i).
\]
Thus, one may solve the strictly convex problem $g_j(\mu, w) \triangleq \inf_{\nu \in \mathcal{B}_j} \sum_{i=1}^K w_i K_i(\mu_i | \nu_i)$
for each $j$ and set  $ g(\mu, w)= \min_{j \leq m} g_j(\mu, w)$. 
An algorithm for numerically solving for $g_j(\mu, w)$ is easily designed and is given in Appendix~\ref{app:polytope}. 
An  
outline of a simple algorithm to compute $C^*(\mu)$ is  as follows: 
\medskip
\newline
(i) Given a $w$, for each $j$, solve the strictly convex optimization problem $\inf_{\nu \in \mathcal{B}_j} \sum_{i=1}^K w_i K_i(\mu_i | \nu_i)$
to determine $g_j(\mu, w)$.
\medskip
\newline
(ii)
Compute $ g(\mu, w)= \min_{j \leq m} g_j(\mu, w)$. Use a numerical procedure to determine the gradient of $g(\mu, w)$ with respect to $w$.
Update $w$ using any version of gradient-descent, and repeat.

\medskip

In Appendix~\ref{subsection_Gaussian},  we restrict ourselves to two arms, both  having a Gaussian distribution with known
and common variance.  This simple setting lends itself to elegant comprehensive analysis and a graphical interpretation.

\section{An asymptotically optimal algorithm}
\label{sec:algo}
In this section, we outline a $\delta$-PAC algorithm (Algorithm \ref{alg:1-param-exp-fam}) for the $\mathcal{PI}$ problem  which, under mild conditions, achieves asymptotically optimal mean termination time as $\delta \to 0$. Both the algorithm and its analysis closely
 follow the best arm identification in \cite{garivier2016optimal}. The sampling rule used in the algorithm (described below) is inspired by the lower bound  Problem \textsf{LB}.  The stopping rule follows from the generalised likelihood ratio method (see \cite{chernoff1959sequential}).
 See  \cite{garivier2016optimal} for the derivation of this rule. 

In Problem \textsf{LB}, let $\mathcal{W}(\mu)$ and $C^*(\mu)$, respectively denote 
the optimal solution set and optimal value. Let $V(\mu, w)$  and $g(\mu, w)$, respectively denote  the optimal solution set and optimal value of the inner sub-problem.
We consider settings where Problem \textsf{LB} has a unique optimal solution. That is,  $|\mathcal{W}(\mu)|=1$. Recall that 
in Problem \textsf{LB}, $\mu \in \mathcal{A}_j$.

\medskip

\noindent {\bf Sampling Rule:}
The essential idea is to draw samples according to estimated optimal sampling ratios obtained by solving Problem \textsf{LB}  with empirical means substituting the true means. In other words, if $\hat{\mu}(t)$ is the vector of empirical means of the arms at time $t$, an arm is chosen to bring the ratio of total number of samples for all the arms closer to an optimal ratio $\hat{w}(t) \in \mathcal{W}(\hat{\mu}(t))$.  But this simple strategy may result in erroneously giving too few samples to an arm due to initial bad estimates preventing convergence to the correct value in subsequent sample allocations. This difficulty can be dealt with through forced exploration for each arm to ensure sufficiently fast convergence. 

\cite{garivier2016optimal} propose a `D-Tracking' rule  along these lines for the best arm problem  that ensures 
convergence to the correct sampling ratio. We also use this rule as the sampling rule in our algorithm. The rule can be described as follows. Let $N_i(t)$ denote the number of samples of arm $i$ at sampling step $t$ for all $i$ and let $\hat{w}(t) \in \mathcal{W}(\hat{\mu}(t))$. \textit{If there exists an arm $i$ such that $N_i(t) < \sqrt{t} - K/2$, choose that arm. Otherwise, choose an arm that has the maximum difference between the estimated optimal ratio and the actual fraction of samples, i.e., an arm is chosen from $\argmax_i \hat{w}_i(t) - N_i(t)/t$.} 
This sampling rule has the following properties: \newline
(i) each arm gets $\Omega(\sqrt{t})$, \newline
(ii)  if the estimated sampling ratios $\mathcal{W}(\hat{\mu}(t))$ converge to an optimal ratio $\mathcal{W}(\mu)$, then the actual fraction of samples also converges to the same optimal ratio.

\medskip 

\noindent {\bf Stopping Rule:}
Let {\em threshold function} $\beta(t, \delta) = \log \left( \frac{c t}{\delta} \right)$, where $c$ is an appropriately chosen constant. 
The stopping rule uses a threshold rule that imitates the lower bound \eqref{eq:lb}. It first finds the partition in which the empirical mean vector $\hat{\mu}(t)$ lies. Denote this partition after generating $t$  samples by $\mathcal{A}(t)$. If
$\inf_{\nu \in \mathcal{A}^c(t)} \sum_i N_i(t) K_i(\hat{\mu}_i(t) | \nu_i)   \geq \beta(t, \delta),$ then it stops and declares  $\mathcal{A}(t)$ as the partition containing  $\mu$. Else, it continues to sample arms according to the D-Tracking rule. 

\begin{algorithm}
  \caption{Algorithm for one parameter exponential families}
  \begin{algorithmic}
  \State Sample each arm once. Set $\hat{\mu}(0)$ to the observed sample average of each arm. Set $t=1$
  \State At sample $t$,
  \State Compute weights $w(\hat{\mu}(t-1))$ and sample according to D-Tracking rule \Comment{Sampling Rule}
  \State Let $\hat{\mu}(t) \in \mathcal{A}(t)$.
  
  \noindent {\bf If} $\inf_{\nu \in \mathcal{A}^c(t)} \sum_i N_i(t) K_i(\hat{\mu}_i(t) | \nu_i)   \geq \beta(t, \delta)$ {\bf then} \Comment{Termination Rule}
  \newline
  Declare $\mu \in \mathcal{A}(t)$.
  \newline
  \noindent {\bf end if}
  \newline
  \noindent {\bf Else} Increment $t$ by 1 and continue.
   \end{algorithmic}
    \label{alg:1-param-exp-fam}
\end{algorithm}

\medskip

\noindent {\bf Sample complexity analysis:}
Let $T_U(\delta)$ be the time at which Algorithm~\ref{alg:1-param-exp-fam} terminates. Then we have the following guarantee.
\begin{theorem}
\label{thm:1-param-exp-fam}
Suppose that $\Omega \subset \mathcal{U}$ and Assumptions 1 and 2 hold. 
If \ref{eq:lb-problem} has a unique optimal solution, i.e., if $\abs{\mathcal{W}(\mu)} = 1$, then Algorithm~\ref{alg:1-param-exp-fam} is a $\delta$-PAC algorithm with
$$\limsup_{\delta \to 0} \frac{\EE[T_U(\delta)]}{\log\left( \frac{1}{\delta} \right)} \leq  T^*(\mu).$$
\end{theorem}
As seen in \cref{sec:lb}, for threshold crossing, half space problem and the polytope
problem  \ref{eq:lb-problem} has a unique optimal solution. When $\mathcal{A}_2$ is a closed convex set and
 the associated $\nu^*\in \partial \mathcal{A}_2$ is a smooth point (with a unique supporting hyperplane), then \ref{eq:lb-problem} again has a unique optimal solution.

The proof of  \cref{thm:1-param-exp-fam} is along the lines of \cite{garivier2016optimal}, and is given in Appendix~\ref{sec:appendix:algo} for completeness.
The  following continuity result is needed for the proof.
\begin{lemma}
\label{lem:cont}
Under conditions of Theorem~\ref{thm:1-param-exp-fam},
the function $g$ is continuous at $(\mu, w)$ for any $w \in \mathcal{P}_K$. Further, if \ref{eq:lb-problem} has a unique optimal solution, then this solution is continuous at $\mu$.
\end{lemma}
For notational ease, let $\mathcal{A}$ denote $\mathcal{A}_j$
and $\mathcal{A}^c$ denote $\Omega -\mathcal{A}_j$.
 If  $\bar{{\mathcal A}^c}$ is compact, then  Theorem 2.1 in \cite{fiacco1990} implies that $g$ is continuous at $(\mu, w)$.
Continuity of  the optimal solution to  \ref{eq:lb-problem}  at $\mu$ when $\mathcal{W}(\mu)$ is a singleton
also follows from Theorem 2.2 in \cite{fiacco1990}. The details for general $\bar{{\mathcal A}^c}$
are given in Appendix~\ref{sec:appendix:algo}.




\appendix
\section{Threshold Crossing Problem} \label{sec:app:threshold}

In this section first  in Example~\ref{ex:1}, we  discuss
how the threshold crossing problem arises naturally  in nested simulation used in financial portfolio risk measurement.
We then prove Theorem~\ref{the:theorem1}.

\begin{example}  \label{ex:1} {\em
Consider the problem of measuring tail risk in a portfolio comprising financial derivatives. The key property of a financial derivative is that as a function of underlying stock prices or other financial instruments, it's value is a conditional expectation (see, e.g., \cite{duffie2010dynamic}, \cite{shreve2004stochastic}). Thus, the value of a portfolio of financial securities that contains financial derivatives can also be expressed as a  conditional expectation given the value of underlying financial instruments.

Suppose that $(X_1, \ldots, X_K)$, where each $X_t$ is a vector in a Euclidean space, denote the macroceconomic variables and financial instruments  at time $t$, such as prevailing interest rates, stock index value and
stock prices, on which the value of a portfolio depends.   For notational convenience we have assumed that times take integer values.

Portfolio loss amount at any time $t$ is a function of ${\mathcal X}_t \triangleq (X_1, \ldots, X_t)$ and is given by
$E(Y_t| {\mathcal X}_t)$ for some random variable $Y_t$ (see, e.g. \cite{gordy2010nested}, \cite{broadie2011efficient}
for further discussion on portfolio loss as a conditional expectation, and the need for  nested simulation).
The quantity $E(Y_t| {\mathcal X}_t)$ is not known, however, conditional on ${\mathcal X}_t$, independent samples of $Y_t$ can be generated via simulation.
Our interest is in estimating the probability that the portfolio loss by time $K$ exceeds a large threshold
$u$ or
\begin{equation} \label{eqn:fin_prob}
\gamma \triangleq P( \max_{1 \leq t \leq K} Z_t \geq u),
\end{equation}
where $Z_t =  E(Y_t|{\mathcal X}_t)$.

These probabilities typically do not have a closed form expression and are estimated using Monte Carlo simulation.
An algorithm to estimate this probability maybe nested
and is given as follows:
\begin{enumerate}
\item
Repeat the outer loop iterations for $1 \leq j \leq n$.
\item
At outer loop iteration $j$, generate through Monte Carlo a sample of underlying factors
$(X_{1,j}, \ldots, X_{K,j})$.
\item
Given this sample, we need to ascertain whether
\[
W_j \triangleq \max_{1 \leq t \leq K} Z_{t,j} \geq u,
\]
where $Z_{t,j} =  E(Y_t|{\mathcal X}_{t,j})$.
 This fits our framework of threshold crossing problem where we may sequentially generate
 conditionally independent samples of $Y_t$ for each $t$ conditional on $(X_{1,j}, \ldots, X_{t,j})$
 and arrive at an indicator $\hat{W}_j$ that equals $W_j$ with probability
 $\geq 1- \delta$.
\end{enumerate}
Then,
\[
\hat{\gamma}_n(\Delta) \triangleq \frac{1}{n}\sum_{j=1}^n \hat{W}_j
\]
denotes our estimator for $\gamma$. There are interesting technical issues related to optimally distributing
 computational budget in deciding
the number of samples in the outer loop, in the inner loop and the value of $\delta$ to be selected.
These issues, however, are not addressed in the paper and may be a topic for future research. }
\end{example}

{\noindent \bf Proof of Theorem \ref{the:theorem1}:}
To see \eqref{eqn:thresh1}, first observe that due to continuity of
each $K_j(\mu_j | \nu_j)$ as a function of $\nu_j \in {\cal U}_j$, we have
\[
\inf_{\nu \in \mathcal{A}_2} \sum_{j=1}^K w_j K_j(\mu_j | \nu_j)
= \inf_{\nu \in \bar{\mathcal{A}}_2} \sum_{j=1}^K w_j K_j(\mu_j | \nu_j),
\]
where recall that for any set $\mathcal{A}$, $ \bar{\mathcal{A}}$ denotes its closure.
The RHS above
is solved by
\[
\nu = (u,\ldots, u, \mu_{i+1}, \ldots, \mu_k)
\]
in the sense that for any other
$\tilde{\nu} \in \bar{\mathcal{A}}_2$,
\[
\sum_{j=1}^K w_j K_j(\mu_j | \tilde{\nu}_j) \geq
\sum_{j=1}^K w_j K_j(\mu_j | {\nu}_j)= \sum_{j=1}^i w_j K_j(\mu_j | u).
\]

Our lower bound problem reduces to
\begin{equation*}
\max_{w \in {\cal P}_K} \sum_{j=1}^i w_j K_j(\mu_j | u).	
\end{equation*}

This can easily be seen to be solved uniquely by
 $w^*_1 = 1$, $w^*_j=0$ for $j=2, \ldots, K$, and the optimal value $C^*$ is $K_1(\mu_1 | u)$. The lower bound on the overall expected number of samples generated is then given by $\log (\frac{1}{2.4 \delta})/C^*$.

\bigskip

To see (\ref{eqn:thresh222}), observe that
to simplify
$\inf_{\nu \in \bar{\mathcal{A}}_1} \sum_{j=1}^K w_j K_j(\mu_j | \nu_j)$, it
suffices to consider $\nu(s) \in \bar{\mathcal{A}_1}$ for each $s$ ($1 \leq s \leq K$) where
\[
\nu(s) \triangleq (\mu_1, \ldots, \mu_{s-1}, u, \mu_{s+1}, \ldots, \mu_k),
\]
 in the sense
that
for any $\nu \in \bar{\mathcal{A}}_1$
\[
 \sum_{j=1}^K w_j K_j(\mu_j | \nu_j) \geq \min_{s=1, \ldots, K} \sum_{j=1}^K w_j K_j(\mu_j | \nu_j(s))=  \min_{s=1, \ldots, K} w_s K_s(\mu_s | u).
\]
The lower bound problem then reduces to
\begin{equation*}
\max_{w \in {\cal P}_K} \min_{j} w_j K_j(\mu_j | u).
\end{equation*}
The solution to this problem is given by $$w^*_j \propto  1/K_j(\mu_j | u) \; \forall j,$$
 and the optimal value $C^*$ is $\left( \sum_{j=1}^K \frac{1}{K_j(\mu_j | u)} \right)^{-1}$. The lower bound on the overall expected number of samples generated is equal to $\log (\frac{1}{2.4 \delta})/C^*$.
$\Box$

\section{The half space lower bound problem}

In this section we restate (to aid readability) Theorem~\ref{thm:Theorem10111}
as Theorem~\ref{thm:Theorem101} and prove it. 
We also  state and prove Lemma \ref{lem:lemma101} needed for
proof of Theorem~\ref{thm:Theorem101}.

\begin{theorem} \label{thm:Theorem101}
Under Assumptions~\ref{ass:kl-cvx}, \ref{ass:kl-range}, and that $\mathcal{A}_2$ is non-empty,
there is a unique optimal solution $(w^*, \nu^*)$ to \ref{eq:lb-problem}.
Further,
\begin{equation} \label{eqn:hs001}
K_i( \mu_i | {\nu}^*_i) = K_1( \mu_1 | {\nu}^*_1)	\quad \forall i,
\end{equation}
\begin{equation} \label{eqn:hs002}
\sum_{k=1}^K a_{k}{\nu}^*_k = b,
\end{equation}
\begin{equation} \label{eqn:hs003}
 {\nu}^*_i> \mu_i \mbox{ if }  a_i >0,
\mbox{ and  }
 {\nu}^*_i < \mu_i \mbox{  if } a_i <0.
\end{equation}
Relations (\ref{eqn:hs001}), (\ref{eqn:hs002}) and (\ref{eqn:hs003}) uniquely
specify ${\nu}^* \in {\mathcal U}$.
Moreover,
\begin{equation} \label{eqn:hs004}
\frac{w^*_i}{a_i}K'_i(\mu_i | {\nu}^*_i) = \frac{w^*_1}{a_1}K'_1(\mu_1 | {\nu}^*_1) \quad \forall i,
\end{equation}
where the derivatives are with respect to the second argument.
\end{theorem}

Let
$\overline{u}_i = \sup\{u \in {\mathcal U}_i\}$, and
$\underline{u}_i = \inf\{u \in {\mathcal U}_i\}$.
Further, set
\[
\hat{u}_i = \overline{u}_i \mbox{ if } a_i>0, \mbox{ and } \hat{u}_i = \underline{u}_i \mbox{ if } a_i<0.
\]

The following lemma is useful in proving Theorem~\ref{thm:Theorem101}.

\begin{lemma} \label{lem:lemma101}
Under Assumption~\ref{ass:kl-range}, the following are equivalent
\begin{enumerate}
\item
$\mathcal{A}_2  \ne \emptyset$.
\item
$\sum_{i=1}^K a_i \hat{u}_i > b$.
\item
There exists a unique ${\nu}^* \in {\mathcal U}$ such that
(\ref{eqn:hs001}), (\ref{eqn:hs002}) and (\ref{eqn:hs003}) hold.
\end{enumerate}
\end{lemma}

\begin{proof}[Proof of Lemma \ref{lem:lemma101}:]

Claim 1 implies existence of $\nu$ such that
$\sum_{i=1}^K a_i \nu_i >b$ and $K_i(\mu_i | \nu_i)< \infty$ for all $i$.
Claim 2 follows as
\[
 \sum_{i=1}^K a_i \nu_i < \sum_{i=1}^K a_i \hat{u}_i.
 \]

To see that Claim 2 implies Claim 3,
recall that $K_i(\mu_i| \nu_i)$  equals zero at
$\nu_i=\mu_i$. It strictly increases with $\nu_i$ for $\nu_i \geq \mu_i$
and it strictly reduces with  $\nu_i$ for $\nu_i \leq \mu_i$.

Assume w.l.o.g. that $a_1>0$, and for $\nu_1 \geq \mu_1$, consider the function
\[
\nu_i(\nu_1) = K_i^{-1}(K_1(\mu_1| \nu_1))
\]
where $\nu_i(\nu_1) \geq \mu_i$ if $a_i>0$, and
$\nu_i(\nu_1) \leq \mu_i$ if $a_i<0$.
Now, the function
\[
h(\nu_1) \triangleq \sum_{i=1}^K a_i \nu_i(\nu_1) < b
\]
for $\nu_1= \mu_1$ and it strictly increases with $\nu_1$.

Further, observe that as $\nu_1 \uparrow \overline{u}_1$,
$\nu_i(\nu_1) \uparrow \overline{u}_i$ if $a_i>0$,
and $\nu_i(\nu_1) \downarrow \underline{u}_i$ if $a_i>0$.
Thus,
$h(\nu_1) \uparrow \sum_{i=1}^K a_i \hat{u}_i$ and
thus there exists a unique
${\nu}^* \in {\mathcal U}$ so that
$h({\nu}^*_1)=b$, and
(\ref{eqn:hs001}) and (\ref{eqn:hs003}) hold.

To see that Claim 3 implies Claim 1,
observe that Claim 3 guarantees that
\[
({\nu}^*_1, \nu_2({\nu}^*_1), \ldots, \nu_K({\nu}^*_1) )\in {\mathcal U}
\]
By selecting $\nu_1 > {\nu}^*_1$ and sufficiently small, Claim 1 follows.
\end{proof}

\bigskip

\noindent{\bf Proof of Theorem \ref{thm:Theorem101}:}

Lemma \ref{lem:lemma101} guarantees the existence of
$\nu^*$, $w^*$ that solve (\ref{eqn:hs001}), (\ref{eqn:hs002}),
(\ref{eqn:hs003}) and (\ref{eqn:hs004}). Here, observe that 
(\ref{eqn:hs004}) defines $w^*$.

Note that  $\nu^*$ is the solution to the optimization problem:
\[
\inf_{\nu \in \bar{\mathcal{A}}_2 } \sum_{j=1}^K w^*_j K_j(\mu_j | \nu_j).
\]
 This can be verified by observing that the first order KKT conditions for this convex programmimg problem are given by 
 (\ref{eqn:hs002}),
(\ref{eqn:hs003}) and (\ref{eqn:hs004}).
  (recall that  $\bar{\mathcal{A}}_2 = \{\nu: \sum_{i=1}^K a_i \nu_i\geq b \}$). Further, from (\ref{eqn:hs001}), it follows that
$$\inf_{\nu \in \bar{\mathcal{A}}_2} \sum_{j=1}^K w^*_j K_j(\mu_j | \nu_j) = \sum_{j=1}^K w^*_j K_j(\mu_j | \nu^*_j) = K_1(\mu_1 | \nu^*_1).$$
For any another feasible solution $\tilde{w}$, we have
\[
\inf_{\nu \in \bar{\mathcal{A}}_2} \sum_{i=1}^K \tilde{w}_i K_i(\mu_i | \nu_i) \leq \sum_{i=1}^K \tilde{w}_i K_i(\mu_i | {\nu}^*_i) \leq K_1(\mu_1 | \nu^*_1),
\]
which shows that $w^*$ is an optimal solution to the problem.

\paragraph*{Uniqueness:} 
It remains to show that  above is a unique 
solution to (\ref{eq:lb-problem}). We skip the details  as, in 
Section~\ref{sec:non-convex}, we prove uniqueness of the solution for a general case where $\bar{\mathcal{A}}_2$ is a union of half-spaces. See Lemma \ref{lem:unique-soln} for uniqueness in this more general setting.
$\Box$

\section{Lower bounds when $\mathcal{A}_2$ is convex}

In this section we prove the results stated in Section~\ref{sec:convex}. 
We first prove Lemma~\ref{lem:unique-vu*}. We then restate Theorem~\ref{thm:min-max11}
as Theorem~\ref{thm:min-max} and prove it.

\bigskip

\noindent {\bf Proof of Lemma~\ref{lem:unique-vu*}:} 
First note that $\inf_{\nu \in \mathcal{A}^c} \sum_{j=1}^K w_j K_j(\mu_j | \nu_j)$ is a concave function of $w$. This shows that, if $w^*$ and $s^*$ are two optimal solutions, then $\alpha w^* + (1-\alpha) s^*$ for $\alpha \in (0,1)$ is another optimal solution.
Since it is optimal, we have
\[
\sum_{j=1}^K (\alpha w^*_j +(1-\alpha) s^*_j) K_j(\mu_j | \nu_j(\alpha w^*+(1-\alpha) s^*)) = C^*.
\]
Now due to Assumption~\ref{ass:kl-cvx},
\[
\sum_{j=1}^K w^*_j K_j(\mu_j | \nu_j(\alpha w^*+(1-\alpha) s^*)) > C^*
\]
if  $\nu(\alpha w^*+(1-\alpha) s^*) \ne \nu(w^*)$
and
\[
\sum_{j=1}^K s^*_j K_j(\mu_j | \nu_j(\alpha w^*+(1-\alpha) s^*)) > C^*
\]
if $\nu(\alpha w^*+(1-\alpha) s^*) \ne \nu(s^*)$,
it follows that $\nu(w^*) = \nu(\alpha w^*+(1-\alpha) s^*) = \nu(s^*)$.
$\Box$

\begin{theorem}
\label{thm:min-max}
Suppose that $\mu \in {\mathcal A}_1$, $\mathcal{A}_2$ is non-empty, $\Omega \subset {\mathcal U}$, and Assumptions  1 and 2 hold. Then, for any optimal solution $(w^*,\nu^*)$  to \ref{eq:conv-set-lb}, the  $\nu^*$ uniquely solves the  min-max problem
\begin{eqnarray}
\label{eq:min-max} 
\inf_{\nu \in \mathcal{A}_2} \max_i  K_i(\mu_i | \nu_i).
\end{eqnarray}

\noindent Further, the following are necessary and sufficient conditions for such an  
$(w^*,\nu^*)$. Let $\mathcal{I} = \argmax_i K_i(\mu_i | \nu^*_i)$. Then,
\begin{enumerate}
\item \label{item:zero} $w^*_i = 0 \quad \forall i \in \mathcal{I}^c,$
\item \label{item:boundary} $\nu^*_{\mathcal I} \in \partial (\mathcal{A}_2)_{\mathcal I}$, and
\item \label{item:kkt} there exists a supporting hyperplane of $(\mathcal{A}_2)_{\mathcal I}$ at $\nu^*_{\mathcal I}$ given by $\sum_{i \in \mathcal I} a_i \nu_i = b$ such that
\begin{gather}
 {\nu}^*_i> \mu_i \mbox{ if }  a_i >0,
\mbox{ and  }
 {\nu}^*_i < \mu_i \mbox{  if } a_i <0 \quad \forall i \in \mathcal{I},	\label{eq:sign-match}	\\
\frac{w^*_i}{a_i} K_i'(\mu_i|\nu^*_i) =  \frac{w^*_j}{a_j} K_i'(\mu_j|\nu^*_j)	\quad \forall i, j \in \mathcal{I}.	\label{eq:kkt-half-space}
\end{gather}
\end{enumerate}
\end{theorem}

\vspace{0.1in}

\begin{proof}[Proof of Theorem~\ref{thm:min-max}]

Let $\mathcal{B}_n$ denote a closed ball centered at $\mu$ with radius $n$. Consider $n$ sufficiently large so that 
$\tilde{\nu}$ defined as the solution to (\ref{eq:min-max}) lies in $\mathcal{B}_n$ (since the objective function $\max_i  K_i(\mu_i | \nu_i)$ is strictly convex in $\nu$, such a $\tilde{\nu}$ is unique).

Since $\mathcal{A}_2 \cap \mathcal{B}_n$ is a compact set, and $\sum_{i=1}^K w_i K_i(\mu_i | \nu_i)$ is continuous in $w$ and $\nu$ and concave in $w \in \mathcal{P}_K$ and convex in $\nu \in \mathcal{A}_2 \cap \mathcal{B}_n$, by Sion's Minimax Theorem
\begin{eqnarray}
\max_{w \in \mathcal{P}_K} \inf_{\nu \in \mathcal{A}_2 \cap \mathcal{B}_n} \sum_{i=1}^K w_i K_i(\mu_i | \nu_i)	& = 
\inf_{\nu \in \mathcal{A}_2 \cap \mathcal{B}_n} \max_{w \in \mathcal{P}_K} \sum_{i=1}^K w_i K_i(\mu_i | \nu_i)	\nonumber \\
 & = \inf_{\nu \in \mathcal{A}_2 \cap \mathcal{B}_n} \max_{i} K_i(\mu_i | \nu_i)	 \nonumber \\
 & = \inf_{\nu \in \mathcal{A}_2} \max_{i} K_i(\mu_i | \nu_i).   \label{eqn:maxmin101}
\end{eqnarray}

Observe that  
\[
r_n(w) \triangleq \inf_{\nu \in \mathcal{A}_2 \cap \mathcal{B}_n} \sum_{i=1}^K w_i K_i(\mu_i | \nu_i)
\] is continuous in $w$ (see Theorem 2.1 in \cite{fiacco1990}) and decreases with  $n$ to 
$r(w) \triangleq \inf_{\nu \in \mathcal{A}_2 } \sum_{i=1}^K w_i K_i(\mu_i | \nu_i)$. 
Thus, we have uniform convergence (see Theorem 7.13 in \cite{rudin76principles})
\[
\sup_{w \in \mathcal{P}_K} |r_n(w) - r(w)| \rightarrow 0.
\]
This in turn implies that 
\[
\max_{w \in \mathcal{P}_K} r_n(w) \rightarrow
\max_{w \in \mathcal{P}_K} r(w).
\]
From  (\ref{eqn:maxmin101}) it follows that LHS above is independent of $n$. Therefore, the min-max relation
\begin{equation} \label{eqn:maxmin102}
\max_{w \in \mathcal{P}_K} \inf_{\nu \in \mathcal{A}_2} \sum_{i=1}^K w_i K_i(\mu_i | \nu_i)
= \inf_{\nu \in \mathcal{A}_2} \max_{i} K_i(\mu_i | \nu_i)
\end{equation}
holds.

Now if $(w^*,\nu^*)$ is a saddlepoint of the min-max problem, and since $\nu^*$ is unique, it equals $\tilde{\nu}$. 

\bigskip

\paragraph*{Necessity of conditions on optimal $(w^*,\nu^*)$:}
Let $\mathcal{I} = \argmax_i K_i(\mu_i | \nu^*_i)$. The minimax equality in
(\ref{eqn:maxmin102})    shows that ($w^*$, $\nu^*$) is a saddle point, and therefore, $w^*$ solves the optimization problem
 \begin{equation}
 \max_{(w_1, \dots, w_K) \in \mathcal{P}_K} \sum_{j=1}^K w_j K_j(\mu_j | \nu^*_j).
\end{equation}
From this, it is easy to see that $w^*_i = 0 \; \forall i \in \mathcal{I}^c$.

To see \ref{item:boundary}, 
note that $\nu^*$ uniquely solves the optimization problem 
\begin{equation}
 \min_{(\nu_1, \dots, \nu_K) \in \mathcal{A}_2} \sum_{j=1}^K w^*_j K_j(\mu_j | \nu_j).
\end{equation}

If $\nu^*_{\mathcal I}$ is in the interior of $(\mathcal{A}_2)_{\mathcal I}$, it is easy to come up with 
$\nu \neq \nu^*$ on $\partial \mathcal{A}_2$, with a smaller value of 
$\sum_{j=1}^K w^*_j K_j(\mu_j | \nu_j)$.

\vspace{0.1in}

Now, consider the convex set $$\mathcal{C} := \left\lbrace \nu_{\mathcal{I}} \in \RR^{\abs{\mathcal{I}}} : \sum_{i \in \mathcal I} w^*_i K_i(\mu_i | \nu_i) < \sum_{i \in \mathcal I} w^*_i K_i(\mu_i | \nu^*_i) \right\rbrace$$ (convexity of $\mathcal{C}$ follows from Assumption \ref{ass:kl-cvx}). By the separating hyperplane theorem, there exists a hyperplane $\sum_{i \in \mathcal I} a_i \nu_i = b$ that separates $\mathcal{C}$ and $(\mathcal{A}_2)_{\mathcal{I}}$. Since $\nu^*_{\mathcal{I}} \in \partial \mathcal{C} \cap \partial (\mathcal{A}_2)_{\mathcal{I}}$, this hyperplane passes through $\nu^*_{\mathcal{I}}$, and is a supporting hyperplane to both convex sets $\mathcal{C}$ and $(\mathcal{A}_2)_{\mathcal{I}}$. From the fact that it is a supporting hyperplane to $\mathcal{C}$ at $\nu^*_{\mathcal{I}}$, we have
$$\frac{w^*_i}{a_i} K_i'(\mu_i|\nu^*_i) = \frac{w^*_j}{a_j} K_i'(\mu_j|\nu^*_j)	\quad \forall i, j \in \mathcal{I}.$$
This proves Condition \ref{item:kkt}.

\bigskip 

\paragraph*{Sufficiency:}
 Let $\nu^*$ and $w^*$ be such that \ref{item:zero}, \ref{item:boundary}, \ref{item:kkt} hold. 
 Note that $\sum_{i \in \mathcal I} a_i \mu_i < b$ and $(\mathcal{A}_2)_{\mathcal I} \subseteq \{\nu_{\mathcal I} : \sum_{i \in \mathcal I} a_i \nu_i \geq b\}$. Then, from Theorem \ref{thm:Theorem101}, $w^*_{\mathcal I}$ and $\nu^*_{\mathcal I}$ solve the following half space problem in the lower dimensional subspace restricted to coordinate set $\mathcal I$:
 \begin{equation*}
 \max_{w_{\mathcal I} \in \mathcal{P}_{\mathcal I}} \inf_{\nu_{\mathcal I}: \sum_{i \in \mathcal I} a_i \nu_i \geq b} \sum_{i \in \mathcal I} w_i K_i(\mu_i | \nu_i).
 \end{equation*}
In particular, $$\inf_{\nu_{\mathcal I}: \sum_{i \in \mathcal I} a_i \nu_i \geq b}  \sum_{i \in \mathcal I} w^*_i K_i(\mu_i | \nu_i) = \sum_{i \in \mathcal I} w^*_i K_i(\mu_i | \nu^*_i).$$
Further, for any $w_{\mathcal I}$, note that $$\inf_{\nu_{\mathcal I} \in (\mathcal{A}_2)_{\mathcal I}} \sum_{i \in \mathcal I} w_i K_i(\mu_i | \nu_i) \geq \inf_{\nu_{\mathcal I}: \sum_{i \in \mathcal I} a_i \nu_i \geq b} \sum_{i \in \mathcal I} w_i K_i(\mu_i | \nu_i).$$ This shows that
$$\inf_{\nu \in \mathcal{A}_2} \sum_{j=1}^K w^*_j K_j(\mu_j | \nu_j) = \inf_{\nu_{\mathcal I} \in (\mathcal{A}_2)_{\mathcal I}} \sum_{i \in \mathcal I} w^*_i K_i(\mu_i | \nu_i) = \sum_{i \in \mathcal I} w^*_i K_i(\mu_i | \nu^*_i) = \max_i K_i(\mu_i | \nu^*_i).$$

Now, consider any $\tilde{w}$ which is a feasible solution of \ref{eq:conv-set-lb}. Then, $$\inf_{\nu \in \mathcal{A}_2} \sum_{i=1}^K \tilde{w}_i K_i(\mu_i | \nu_i) \leq \sum_{i=1}^K \tilde{w}_i K_i(\mu_i | \nu^*_i) \leq \max_i K_i(\mu_i | \nu^*_i).$$
This proves our claim that $w^*$, $\nu(w^*) = \nu^*$ form an optimal solution.
\end{proof}

\section{Lower bound analysis when ${\cal A}_1$ is a polytope}
\label{app:polytope}

In this section we first outline a simple algorithm to solve the sub problem
$$g_j(\mu, w) = \inf_{\nu \in \mathcal{B}_j} \sum_{i=1}^K w_i K_i(\mu_i | \nu_i)$$  as discussed in Section~\ref{sec:non-convex}.
We then provide a proof of Lemma~\ref{lem:unique-soln}.
In Appendix~\ref{subsection_Gaussian},  we consider two arms, both  with Gaussian distribution and known
and common variance.  In this simple setting, we conduct a comprehensive analysis of the lower bound problem
and provide a graphical interpretation of the solutions.

\bigskip

\noindent {\bf Solving $g_j(\mu, w)$:} Observe that solving for $g_j(\mu, w)$ is equivalent to solving
\begin{equation}  \label{eqn:app:last}
\inf_{\nu: \sum_{i=1}^K a_i \nu_i \geq b} \sum_{i=1}^K w_i f_i(\nu_i),
\end{equation}
for a given $w \in \mathcal{P}_K$, where each $f_i$ is strictly convex, and $f_i(\mu_i)=f_i'(\mu_i)=0$,
and without loss of generality $b>0$. 
Again, without loss of generality, we assume that $w_i >0$ for each $i$.
The existence of a unique solution is best  seen from the graphical interpretation 
in Remark~\ref{rem:33}. We now discuss how this may be efficiently computed.

Observe that
$f'_i$ is a strictly increasing function. Let 
$h_i$ denote the inverse function of $f'_i$. $h_i$ is also strictly increasing. 

The first order conditions applied to (\ref{eqn:app:last}) imply that the optimal solution $\nu^*$ satisfies
\[
\nu_i^* = h_i(\frac{\lambda a_i}{w_i}
\]
for a non-negative  $\lambda$ such that
\begin{equation}  \label{eqn:app:last2}
 \sum_{i=1}^K a_i h_i(\frac{\lambda a_i}{w_i}) = b.
 \end{equation}
Observe that $\sum_{i=1}^K a_i h_i(\frac{\lambda a_i}{w_i})$ equals 0 for $\lambda =0$,
and it strictly increases with increase in $\lambda$.
Thus one may use any line search method to find
$\lambda$ that solves (\ref{eqn:app:last2}).

\bigskip

\noindent {\bf Proof of Lemma~\ref{lem:unique-soln}:} 
Denote the optimal value of \eqref{eq:union-lb} by $C^*$. We first show that if $q, s  \in \mathcal{P}_K$ are two distinct optimal solutions and $\nu(q), \nu(s) \in \mathcal{A}^c$, respectively achieve the minimum in the sub-problem, then $\nu(q) \neq \nu(s)$. To see this, suppose $\nu(q) = \nu(s) = \nu \in \partial \mathcal{B}_j$ for some $1 \leq j \leq m$. Then $\nu$ achieves the minimum in the subproblem $\inf_{\nu \in \mathcal{B}_j} \sum_{i=1}^K w_i K_i(\mu_i | \nu_i)$ for both $w = q$ and $w = s$. Hence, both $q, s$ solve the following equations:
\begin{equation}
\sum_{i=1}^K w_i K_i(\mu_i | \nu_i) = C^*,
\end{equation}
\begin{equation}
\frac{w_i}{a_{j,i}}K'_i(\mu_i | {\nu}_i) = \frac{w_1}{a_{j,1}}K'_1(\mu_1 | {\nu}_1) \quad \forall i.
\end{equation}
This is a contradiction as the above set of equations has a unique solution.

Now, suppose $q, s  \in \mathcal{P}_K$ are two distinct optimal solutions of the convex program \eqref{eq:union-lb}. Then any convex combination $z = \alpha q + (1-\alpha)s$ is also an optimal solution. Let $\nu(z)$ achieve the minimum in the sub-problem for $z$. Then
\begin{align*}
C^* = \sum_{i=1}^K z_i K_i(\mu_i | \nu_i(z)) = \alpha \sum_{i=1}^K q_i K_i(\mu_i | \nu_i(z)) + (1-\alpha) \sum_{i=1}^K s_i K_i(\mu_i | \nu_i(z)).
\end{align*}
In addition, for any $\nu$, we have $\sum_{i=1}^K w_i K_i(\mu_i | \nu_i) \leq C^*$ for both $w = q$ and $w = s$. Then, the above equality is possible only if $\sum_{i=1}^K q_i K_i(\mu_i | \nu_i(z)) = \sum_{i=1}^K s_i K_i(\mu_i | \nu_i(z)) = C^*$. This in turn implies that $\nu(z)$ achieves the minimum in the sub-problem for both $q, s$, which is a contradiction to our earlier result. Hence proved.
$\Box$

\subsubsection{Two arms Gaussian setting} \label{subsection_Gaussian}

To illustrate the issues that arise with $\mathcal{A}_2$ being a union of half-spaces, consider a simple setting of two arms.
Both are assumed to have a Gaussian distribution and the  variance of each arm is assumed to be 1/2.  W.l.o.g.
mean of each arm is set to zero.
Then, for $j=1,2$,
\begin{equation}
\mathcal{B}_j = \{\nu \in \RR^2: a_{j,1}\nu_1 +  a_{j,1}\nu_2 \geq b_j \},
\end{equation}
and $\mathcal{A}_2 =  \mathcal{B}_1  \cup \mathcal{B}_2$ be the union of the two half-spaces.
To avoid degeneracies we  assume that each $a_{j,k} \ne 0$.
Further suppose that
 $\frac{a_{1,1}}{a_{1,2}} \ne \frac{a_{2,1}}{a_{2,2}}$ so that $\mathcal{A}_2$ is non-convex.

The lower bound problem is then given by
 \begin{equation}
 \label{eq:eg-2hs}
\max_{(w_1,w_2) \in \mathcal{P}_2}  \inf_{\nu \in \mathcal{A}_2} \sum_{i=1}^2 w_i \nu_i^2.
 \end{equation}

 The following geometrical result provides useful insights towards solution of \eqref{eq:eg-2hs}.

\begin{proposition} \label{prop:prop1}
For $w_1, w_2, C>0$, a necessary and sufficient condition for an ellipse of the form
\begin{equation} \label{eqn:ellipse1}
w_1 \nu_1^2 + w_2 \nu_2^2=C
\end{equation}
to be uniquely tangential to lines
\begin{equation} \label{eqn:plane1}
a_{1,1} \nu_1 + a_{1,2} \nu_2 = b_1
\end{equation}
and
\begin{equation} \label{eqn:plane2}
a_{2,1} \nu_1 + a_{2,2} \nu_2 = b_2
\end{equation}
 is that
\begin{equation} \label{eqn:param_cond}
\min_{k=1,2} |\frac{a_{2,k}}{a_{1,k}}| < \frac{b_2}{b_1} <
\max_{k=1,2} |\frac{a_{2,k}}{a_{1,k}}|.
\end{equation}
Then, the tangential ellipse is specified by
\begin{equation} \label{eqn:param_cond1}
\frac{w_1}{C}= \frac{(a_{1,2} a_{2,1})^2 - (a_{1,1} a_{2,2})^2}{(b_2 a_{1,2})^2 - (b_1 a_{2,2})^2}
\end{equation}
and
\begin{equation} \label{eqn:param_cond2}
\frac{w_2}{C}= \frac{(a_{1,2} a_{2,1})^2 - (a_{1,1} a_{2,2})^2}{(b_1 a_{2,1})^2 - (b_2 a_{1,1})^2}.
\end{equation}
The ellipse (\ref{eqn:ellipse1}) meets the line (\ref{eqn:plane1}) at point
\[
 \left (\frac{C a_{1,1} }{w_1 b_1}, \frac{C a_{1,2} }{w_2 b_1} \right )
 \]
 and it meets line  (\ref{eqn:plane2})
 at point
 \[
 \left (\frac{C a_{2,1} }{w_1 b_2}, \frac{C a_{2,2} }{w_2 b_2} \right).
 \]
\end{proposition}
\begin{proof}
A necessary and sufficient condition for ellipse (\ref{eqn:ellipse1}) to be tangential
to line (\ref{eqn:plane1}) at point $(\nu^*_1, \nu^*_2)$ is for
$(\nu^*_1, \nu^*_2)$ to satisfy the two equations of ellipse and the line, respectively, and the slope matching condition
\begin{equation} \label{eqn:slopematch1}
\frac{w_1 \nu^*_1}{a_{1,1}} = \frac{w_2 \nu^*_2}{a_{1,2}}.
\end{equation}
The fact that $(\nu^*_1, \nu^*_2)$ satisfies   (\ref{eqn:ellipse1}) and (\ref{eqn:plane1})
implies that (\ref{eqn:slopematch1})  equals  $C/b_2$.
Plugging $(\nu^*_1, \nu^*_2)$ from (\ref{eqn:slopematch1}) into
(\ref{eqn:ellipse1}), we observe,
\[
\frac{a_{1,1}^2}{w_1}+\frac{a_{1,2}^2}{w_2} = \frac{b_1^2}{C}.
\]
Similarly, considering the other half-space, we get
\[
\frac{a_{2,1}^2}{w_1}+\frac{a_{2,2}^2}{w_2} = \frac{b_2^2}{C}.
\]
The result follows by solving the two equations.
\end{proof}

\begin{theorem}
The solution to  \eqref{eq:eg-2hs} depends in the following way on the underlying parameters

{\noindent Case 1:}
\begin{equation}
\label{eq:case1}
\left (  \frac{b_2}{b_1}  \right )^2 \geq \left ( \frac{a_{2,1}^2}{|a_{1,1}|} +  \frac{a_{2,2}^2}{|a_{1,2}|} \right ) (|a_{1,1}| + |a_{1,2}|)^{-1}.
\end{equation}
In this case, \eqref{eq:eg-2hs} reduces to the half-space problem where $\mathcal{A}_2 = \mathcal{B}_1$ so that the optimal solution to \eqref{eq:eg-2hs} is given by
\begin{equation}  \label{eqn:total_contr58}
w^*_i = \frac{|a_{1,i}|}{|a_{1,1}|+|a_{1,2}|}, \quad i = 1, 2,
\end{equation}
and the optimal value $C^* = \frac{b_1^2}{(|a_{1,1}| + |a_{1,2}|)^2}.$

{\noindent Case 2:}
\[
\left (  \frac{b_2}{b_1}  \right )^2 \leq \left ( \frac{a_{1,1}^2}{|a_{2,1}|} +  \frac{a_{1,2}^2}{|a_{2,2}|} \right )^{-1} (|a_{2,1}| + |a_{2,2}|).
\]
This simply corresponds to Case 1, with the $(a_{1,1}, a_{1,2},b_1)$ interchanged with $(a_{2,1}, a_{2,2},b_2)$.

{\noindent Case 3:}
\begin{equation} \label{eqn:param_condanother}
 \left ( \frac{a_{1,1}^2}{|a_{2,1}|} +  \frac{a_{1,2}^2}{|a_{2,2}|} \right )^{-1} (|a_{2,1}| + |a_{2,2}|) <  \left (  \frac{b_2}{b_1}  \right )^2
< \left ( \frac{a_{2,1}^2}{|a_{1,1}|} +  \frac{a_{2,2}^2}{|a_{1,2}|} \right ) (|a_{1,1}| +| a_{1,2}|)^{-1}.
\end{equation}
Here  (\ref{eqn:param_cond}) holds, and the optimal
$w^*_1$ and
$w^*_2$ are given by (\ref{eqn:param_cond1}) and (\ref{eqn:param_cond2}), respectively.
\end{theorem}
\begin{proof}
{\noindent \bf Case 1:}

First consider the half-space problem where $\mathcal{A}_2 = \mathcal{B}_1$. Our analysis
in \ref{sec:half-space} shows that there is a unique $(w^*_1, w^*_2)$ and $(\nu^*_1, \nu^*_2)$
that solves the resulting problem, and
\[
\sign (a_{1,1}) \nu_1^* = |\nu_1^*|= \sign (a_{1,2}) \nu_2^*=  |\nu_2^*|,
\]
$a_{1,1} \nu_1^* + a_{1,2} \nu_2^*=b_1$ so that
\[
|\nu_1^*| = |\nu_2^*| = \frac{b_1}{|a_{1,1}| + |a_{1,2}|}.
\]
Further, from
\begin{equation}  \label{eqn:total_contr56}
\frac{w_1^* \nu_1^*}{a_{1,1}} = \frac{w_2^* \nu_2^*}{a_{1,2}},
\end{equation}
it follows that for the half-space problem, $w^*_i \propto |a_{1,i}|$ is the optimal solution and the optimal value
 $C^* = \frac{b_1^2}{(|a_{1,1}| + |a_{1,2}|)^2}.$

Returning to \eqref{eq:eg-2hs}, we show that when \eqref{eq:case1} is true and
and  $w^*_i \propto |a_{1,i}|$,
$$\inf_{\nu \in \mathcal{B}_2} \sum_{i=1}^2 w^*_i K_i(\mu_i | \nu_i) =
 \inf_{\nu: a_{2,1} \nu_1 + a_{2,2} \nu_2 \geq b_2} w^*_1 \nu_1^2 + w^*_2 \nu_2^2 \geq C^*$$ and hence  $w^*_i \propto |a_{1,i}|$ continues to be optimal for \eqref{eq:eg-2hs}.

We first find the point $(\kappa^*_1, \kappa^*_2) \in \mathcal{B}_2$ that achieves the minimum in the above optimization problem. We know that $(\kappa^*_1, \kappa^*_2)$  satisfies
\[
a_{2,1} {\kappa^*_1} +  a_{2,2} {\kappa^*_2} = b_2,
\]
and the slope matching condition
\[
\frac{w^*_1 \kappa^*_1}{a_{2,1}} = \frac{w^*_2 \kappa^*_2}{a_{2,2}}.
\]
It follows from easy calculations that
$$\inf_{\nu: a_{2,1} \nu_1 + a_{2,2} \nu_2 \geq b_2} w^*_1 \nu_1^2 + w^*_2 \nu_2^2 = \frac{b_2^2}{\frac{a_{2,1}^2}{w^*_1}+\frac{a_{2,2}^2}{w^*_2}} =
 \frac{b_2^2}{\left( \frac{a_{2,1}^2}{|a_{1,1}|}+\frac{a_{2,2}^2}{|a_{1,2}|} \right)(|a_{1,1}| + |a_{1,2}|)}.$$
The above expression is greater than $\frac{b_1^2}{(|a_{1,1}| + |a_{1,2}|)^2}$ when \eqref{eq:case1} is true, which gives us the required result.

\vspace{0.1in}

{\noindent \bf Case 2:}
Case 2 follows similarly as Case 1.

\vspace{0.1in}

{\noindent \bf Case 3:}

It is easy to see that (\ref{eqn:param_condanother}) implies
(\ref{eqn:param_cond}).

Let $(w^*_1, w^*_2)$ denote the optimal solution to \eqref{eq:eg-2hs}.
It is clear that the corresponding ellipse
must be tangential
to both the half lines $a_{1,1} \nu_1 + a_{1,2} \nu_2 =b_1$
and $a_{2,1} \nu_1 + a_{2,2} \nu_2 =b_2$, since if it does not touch one of these half lines,
then the associated constraint can be ignored in solving \eqref{eq:eg-2hs}. However,
that violates (\ref{eqn:param_condanother}).

Therefore, the solution is provided by Proposition \ref{prop:prop1}.
\end{proof}

\section{Analysis related to the proposed algorithm}
\label{sec:appendix:algo}

In this section, we first prove Lemma~\ref{lem:cont}. Then,
in Lemma~\ref{lem:d-track} we summarize results from \cite{garivier2016optimal}
on the D-tracking rule that we use in our proof of \cref{thm:1-param-exp-fam}.
This proof is more or less identical to that in \cite{garivier2016optimal}. We keep it here 
for completeness.

\bigskip

\noindent {\bf Proof of Lemma~\ref{lem:cont}:}  
 First suppose  that $\bar{{\mathcal A}^c}$ is compact. The fact that $g$ is continuous at $(\mu, w)$
 follows from the continuity results for non-linear programs. Specifically,
 since the objective function is continuous in $\nu$ and $\bar{{\mathcal A}^c}$ is compact, 
 Theorem 2.1 in \cite{fiacco1990} implies that $g$ is continuous at $(\mu, w)$.

Now consider non-compact  $\bar{{\mathcal A}^c}$ and define
\[
g_n(\mu, w) = \inf_{\nu \in \bar{{\mathcal A}^c} \cap \mathcal{B}_n} \sum_{i=1}^K {w}_i K_i({\mu}_i | \nu_i)
\]
for each $n$ where $\mathcal{B}_n$ is an Euclidean closed ball of radius $n$ centred at $\mu$.
$n$  is taken to be sufficiently large so that $\bar{{\mathcal A}^c} \cap \mathcal{B}_n$ is non-empty.

Then, $g_n(\mu, w)$ is continuous in $(\mu, w)$ and decreases with $n$ to
$g(\mu, w)$. Since this convergence is uniform, it follows that $g(\mu, w)$
is continuous in $(\mu, w)$. 

To see that the optimal solution to  \ref{eq:lb-problem} is continuous at $\mu$ if $\mathcal{W}(\mu)$ is a singleton, note that the problem is equivalent to $\max_{w \in \mathcal{P}_K} g(\mu, w)$. Since $g(\mu, \cdot)$ is continuous on $\mathcal{P}_K$ and $\mathcal{W}(\mu)$ is a singleton, from Theorem 2.2 in \cite{fiacco1990}, we conclude that the optimal solution is continuous at $\mu$.
$\Box$

\bigskip

\begin{lemma}
\label{lem:d-track}
The D-tracking rule ensures that $\min_i N_i(t) \geq \left( \sqrt{t} - K/2 \right)^+ - 1$ and that for all $\epsilon > 0$, for all $t_0$, there exists $t_{\epsilon} \geq t_0$ such that if 
$\sup_{t \geq t_0} \max_i \babs{\hat{w}_i(t) - w_i} \leq \epsilon$ for some $w \in \mathcal{P}_K$, then
$$\sup_{t \geq t_{\epsilon}} \max_i \babss{\frac{N_i(t)}{t} - w_i} \leq 3 (K-1) \epsilon.$$
\end{lemma}

\bigskip

\noindent {\bf  Proof of \cref{thm:1-param-exp-fam}:}
Recall that  $\mu \in \mathcal{A}$. We first prove that the probability of error is at most $\delta$.
\begin{align*}
\PP_{\mu}[\text{error}] & \leq \PP_{\mu}\left[ \exists t \geq 1 : \inf_{\nu \in \mathcal{A}} \sum_i N_i(t) K_i(\hat{\mu}_i(t) | \nu_i)  \geq \beta(t, \delta);
\hat{\mu}_t \in \mathcal{A}^c \right]	\\
 & \leq \sum_{t=1}^{\infty}\PP_{\mu}\left[ \sum_i N_i(t) K_i(\hat{\mu}_i(t) | \mu_i)  \geq \beta(t, \delta) \right]	\\
 & \leq \sum_{t=1}^{\infty} e^{K+1}  \left( \frac{\beta(t, \delta)^2 \log t}{K} \right)^K e^{-\beta(t, \delta)}	\\
 & \leq \delta
\end{align*}
if $c$ is chosen large enough s.t.
\begin{align*}
\sum_{t=1}^{\infty} \frac{e^{K+1}}{c t}  \left( \frac{\log^2(c t) \log t}{K} \right)^K \leq 1.
\end{align*} 

The third inequality above follows from \cite{magureanu2014Lipschitz} extended from Bernoulli family  to SPEF. 

\vspace{0.1in}

Next, we prove the upper bound on the mean termination time. 
Let 
$\mathcal{B}_{\infty}^{y}(x)$  denote a Euclidean ball around $x$ of length $y$ under the $\max$ norm.

Fix an $\epsilon > 0$. From the continuity of  $w$ at $\mu$, there exists $\xi > 0$ such that for any $\mu' \in \mathcal{B}_{\infty}^{\xi}(\mu)$ we have $w(\mu') \in \mathcal{B}_{\infty}^{\epsilon}(w(\mu))$. For any $T \in \NN$, define the event 
$\mathcal{E}_T := \bigcap_{t=h(T)}^T \{\hat{\mu}(t) \in \mathcal{B}_{\infty}^{\xi}(\mu)\}.$ It is easy to show that (see Lemma~19 of \cite{garivier2016optimal}) there exist constants $B$, $C$ depending on $\epsilon$ and $\mu$ such that 
$$\PP_{\mu}\left[ \mathcal{E}_T^c \right] \leq B \exp\left( -C T^{1/8} \right).$$
Note that $\xi$, $\mathcal{E}_T$, $B$, $C$ are all functions of $\epsilon$ and $\mu$.

Now, for every $\epsilon > 0$, define $$C^*_{\epsilon}(\mu) = \inf_{\substack{\mu' \in \mathcal{B}_{\infty}^{\xi(\epsilon)}(\mu),\\ w' \in \mathcal{B}_{\infty}^{3(K-1)\epsilon}(w(\mu))}} g(\mu', w').$$
By the continuity of $w$ and $g$, we have $$\lim_{\epsilon \to 0} C^*_{\epsilon}(\mu) = C^*(\mu) = (T^*(\mu))^{-1}.$$ 
From Lemma~\ref{lem:d-track}, for any $\epsilon > 0$, we have for every $T \geq T_{\epsilon}$ that on $\mathcal{E}_T(\epsilon)$, 
$$\Bnorm{\frac{N(t)}{t} - w(\mu)}_{\infty} \leq 3(K-1)\epsilon \quad \forall t > \sqrt{T},$$
which in turn implies that
$$g\left( \hat{\mu}(t), \frac{N(t)}{t} \right) \geq C^*_{\epsilon}(\mu) \quad \forall t > \sqrt{T}.$$
Since the termination rule in the algorithm is given by
$$g\left( \hat{\mu}(t), \frac{N(t)}{t} \right) \geq \frac{\beta(t, \delta)}{t},$$ 
for $T \geq T_{\epsilon}$, on $\mathcal{E}_T(\epsilon)$, we have
\begin{eqnarray*}
\min(T_U(\delta), T) & \leq \sqrt{T} + \sum_{t = \sqrt{T}}^T \ind{C^*_{\epsilon}(\mu) < \frac{\beta(t, \delta)}{t}}	\\
 & \leq \sqrt{T} + \frac{\beta(T, \delta)}{C^*_{\epsilon}(\mu)}.
\end{eqnarray*}
Now, let 
$$T_0(\delta) := \inf \left\lbrace T \in \NN : \sqrt{T} + \frac{\beta(T, \delta)}{C^*_{\epsilon}(\mu)} < T \right\rbrace.$$
Therefore, for any $T \geq \max\{T_{\epsilon}, T_0(\delta)\},$ on $\mathcal{E}_T(\epsilon)$, we have $T_U(\delta) < T,$ which gives us
$$\EE[T_U(\delta)] \leq \sum_{T=1}^{\infty} \PP[T_U(\delta) > T] \leq \max\{T_{\epsilon}, T_0(\delta)\} + \sum_{T=1}^{\infty} \PP[\mathcal{E}_T(\epsilon)^c].$$
As shown in  \cite{garivier2016optimal}, we have
$$T_0(\delta) = \frac{1}{C^*_{\epsilon}(\mu)}\left( O\left( \log \left( 1/\delta \right) \right) + o\left( \log \log \left( 1/\delta \right) \right) \right).$$
This gives us
$$\limsup_{\delta \to 0} \frac{\EE[T_U(\delta)]}{\log\left( \frac{1}{\delta} \right)} \leq \frac{1}{C^*_{\epsilon}(\mu)}.$$
Now, letting $\epsilon$ go to zero, we get
$$\limsup_{\delta \to 0} \frac{\EE[T_U(\delta)]}{\log\left( \frac{1}{\delta} \right)} \leq \lim_{\epsilon \to 0} \frac{1}{C^*_{\epsilon}(\mu)} =  T^*(\mu).$$
%
$\Box$

\end{document}